\documentclass{article}

\usepackage[preprint]{neurips_2026}

\usepackage[utf8]{inputenc} 
\usepackage[T1]{fontenc}    
\usepackage[table]{xcolor}
\definecolor{dark-blue}{RGB}{0,0,191}
\usepackage[hidelinks,colorlinks=true,citecolor=dark-blue,linkcolor=dark-blue,urlcolor=dark-blue, backref=page]{hyperref}
\usepackage{url}            
\usepackage{booktabs}       
\usepackage{amsfonts}       
\usepackage{nicefrac}       
\usepackage{microtype}      
\usepackage{xcolor}         

\title{Robust Data-Collection Policy Learning for Low-Variance Online Policy Evaluation}

\usepackage{pifont}
\newcommand{\cmark}{\ding{51}}
\newcommand{\xmark}{\ding{55}}

\usepackage{amsmath}
\usepackage{amssymb}
\usepackage{mathtools}
\usepackage{amsthm}
\usepackage{color}
\mathtoolsset{showonlyrefs}

\usepackage{algorithmic}

\usepackage{enumitem}

\usepackage[capitalize,noabbrev]{cleveref}

\usepackage{xurl} 
\usepackage{listings}
\usepackage{xcolor}
\usepackage{thmtools}

\usepackage{pifont}
\theoremstyle{plain}
\newtheorem{theorem}{Theorem}[section]

\theoremstyle{definition}
\newtheorem{definition}[theorem]{Definition}

\theoremstyle{remark}

\usepackage{algorithm}

\newcommand{\V}{\mathbb{V}}

\DeclareMathOperator*{\argmin}{argmin}
\DeclareMathOperator*{\argmax}{argmax}

\newcommand{\field}[1]{\mathbb{#1}}

\newcommand{\E}{\field{E}}

\newcommand{\norm}[1]{\left\|{#1}\right\|}

\newcommand{\KL}{{\text{\rm KL}}}

\newcommand{\fA}{\mathcal{A}}

\newcommand{\OPE}{\mathrm{OPE}}
\newcommand{\IS}{\mathrm{IS}}
\newcommand{\ISfunction}{\IS(\pi_\theta, p_\omega,H)}

\newcommand{\gradomega}{\frac{\partial}{\partial \omega}}
\newcommand{\Hsample}{H\sim p_\omega,\pi_\theta}
\newcommand{\Hsampleoff}{H\sim p_{\omega'}}

\newcommand{\explain}[1]{\tag*{(#1)}}

\usepackage[textsize=tiny]{todonotes}
\usepackage{physics}

\author{%
  Claire Chen\\
California Institute of Technology\\
\texttt{clairechen@caltech.edu} \\
  \And
Shuze Daniel Liu\\
Massachusetts Institute of Technology\\
\texttt{shuzel@mit.edu}\\
Purdue University\\
\texttt{daniel.liu@purdue.edu}\\
  \And 
  Licheng Luo\\
University of California, Riverside\\
\texttt{lichengl@ucr.edu}\\
    \And
  Rohan Chandra\\
   University of Virginia\\
  \texttt{rohanchandra@virginia.edu} \\
\And
Nan Jiang\\
University of Illinois Urbana-Champaign \\
\texttt{nanjiang@illinois.edu}\\
\And
  Shangtong Zhang\\
   University of Virginia\\
  \texttt{shangtong@virginia.edu} \\
}

\begin{document}

\maketitle
\begin{abstract}
In reinforcement learning policy evaluation, classic on-policy methods often suffer from high variance when estimating policy performance. To mitigate this issue, 
\textit{behavior policy search} has been proposed to learn data-collecting policies tailored to reduce online evaluation variance. 
However, these approaches do not account for uncertainties in the transition functions. In practice, simulator transitions often differ from the real world due to modeling errors or approximation limitations. As a result, behavior policies trained in simulation may still yield high variance when deployed in real environments, leading to costly reliance on real-world evaluation samples.
In this work, we propose a double-loop gradient-based algorithm for learning behavior policies that are both efficient and robust to transition uncertainty. Theoretically, we derive novel transition-variance gradient expressions and establish global convergence guarantees for the algorithm. Numerically, we demonstrate that our method is less sensitive to transition perturbations than existing approaches, providing supportive evidence for its practical utility.
\end{abstract}

\section{Introduction}
Reinforcement learning (RL) has achieved remarkable success in recent years across domains such as robotics, healthcare, recommendation systems, and natural language processing \citep{mnih2015human, silver2017mastering,jumper2021highly, xie2026beyond, liu2026strategic, liu2026instructing}. A central component in these advances is policy evaluation, the task of estimating the performance of a policy. The most direct approach is the on-policy Monte Carlo method, which collects trajectories from the target policy and estimates its value by averaging the observed returns. Although conceptually simple and widely used, this method often suffers from high evaluation variance, limiting the reliability of the resulting estimates.

To improve efficiency, a growing body of work has investigated learning a separate data-collecting behavior policy to reduce evaluation variance via off-policy evaluation \citep{hanna2017data, zhong2022robust, liu2024efficient, liu2024doubly, liu2025efficientrobust}. This line of research, known as behavior policy search (BPS), optimizes a variance-reducing behavior policy so that the collected trajectories yield more informative evaluations. With an appropriately chosen behavior policy, BPS has been shown to achieve lower variance than naive on-policy evaluation. By contrast, in the standard formulation of off-policy evaluation (OPE), the data are assumed to be pre-logged by a fixed behavior policy and the focus is on designing improved estimators. In comparison, BPS explicitly optimizes the behavior policy itself to reduce variance and is broadly applicable across different OPE estimators.

Despite this progress, existing BPS 
methods generally optimize behavior policies under prespecified transition functions, without accounting for underlying uncertainties. In practice, the true transition functions often deviate from the assumed model due to approximation errors, adversarial perturbations, or partial observability. Such discrepancies can compromise evaluation reliability: a behavior policy optimized under simulator transition may still yield high variance when deployed in the real environment. Consequently, prior methods may continue to require a large number of costly real-world samples to achieve accurate evaluation.

To address these two challenges—variance reduction and transition mismatch—we propose an \textbf{efficient and robust policy evaluation} framework that explicitly accounts for transition uncertainty. Our method formulates BPS as a minimax optimization problem, where an adversarial transition model seeks to maximize the evaluation variance while the behavior policy is optimized to minimize it. Our contributions are summarized as follows: \textbf{(1)} We introduce a novel adversarial framework for robust behavior policy search in policy evaluation (Section~\ref{sec: adversarial minmax}); \textbf{(2)} We derive analytical transition-gradient expressions for both on-transition and off-transition settings, and provide convergence guarantees for the inner-loop adversarial optimization (Section~\ref{sec: inner loop}); \textbf{(3)} We propose a double-loop robust gradient algorithm and provide global convergence guarantee for variance-minimizing behavior policy search (Section~\ref{sec: outer loop}); \textbf{(4)} We numerically show that our method is less sensitive to transition perturbations than existing approaches (Section~\ref{sec: empirical}), verifying the utility of our theoretical results. 
\vspace{-0.75em}
\section{Related Work}
\textbf{Behavior Policy Search.}
Behavior policy search (BPS) reduces evaluation variance by optimizing the data-collecting policy. \citet{hanna2017data} formulate this as an optimization problem, using stochastic gradient descent to outperform standard Monte Carlo methods. \citet{zhong2022robust} extend this via adaptive behavior policies that prioritize under-sampled regions. However, these approaches overlook transition uncertainty; consequently, behavior policies optimized in simulation may still induce high variance under real-world dynamics. In contrast, our method explicitly models transition uncertainty to ensure the learned behavior policy remains effective even under perturbed dynamics.

\citet{liu2024efficient} also study the variance reducing problem, deriving a closed-form, offline-learnable behavior policy with theoretical guarantees. However, their formulation relies on pre-logged data with prespecified and fixed transition probabilities, and cannot adapt when these probabilities shift at deployment, leaving it vulnerable to modeling errors and sim-to-real mismatches. Our method fills this gap by integrating adversarial transition modeling into behavior policy search, combining efficiency with robustness to transition shifts. Similarly, while \citet{russo2025adaptive} optimize adaptive exploration for multi-policy evaluation, they assume fixed dynamics; in contrast, our framework proactively targets robustness to adversarial transition shifts.

\textbf{Robust Policy Evaluation.} Recent work has begun addressing robustness in reinforcement learning policy evaluation. For example, \citet{katdare2023marginalized} and \citet{voloshin2021minimax} propose techniques to improve robustness under simulator mismatch via estimator modification or robust model learning. However, these methods either rely on access to real-world data or focus on minimizing worst-case prediction errors, and do not directly address the high-variance issue central to policy evaluation. In contrast, our work proactively reduces variance by designing behavior policies that are robust to adversarial transition shifts, without requiring target-environment data.

While robust MDP (RMDP) frameworks \citep{iyengar2005robust,nilim2005robust} also consider robustness under transition uncertainty, they are typically designed for reward maximization and rely on linear programming techniques. These approaches (e.g., \citet{wang2023policy,wang2024policy}) do not apply to our setting, where the goal is to minimize the variance of policy value estimators, a fundamentally non-linear objective. We fill this gap by proposing a novel adversarial transition variance gradient method that explicitly targets variance reduction under transition uncertainty.

\section{Background}
\subsection{Markov Decision Process}
We study a finite-horizon Markov Decision Process (MDP, \citet{puterman2014markov}) with finite state space $\mathcal{S}$ and finite action space $\fA$. 
For a finite set $\mathcal{X}$, we denote the probability simplex over $\mathcal{X}$ by $\Delta(\mathcal{X})\doteq\{p:\mathcal{X}\to [0,1]\mid \sum_{x\in \mathcal{X}} p(x)=1\}$.
The MDP consists of a transition probability function
$p:\mathcal{S}\times\mathcal{A}\to\Delta(\mathcal{S})$,
a reward function $r:\mathcal{S}\times\fA\to[0,1]$,
an initial state distribution $p_0\in\Delta(\mathcal{S})$,
and a fixed horizon length $T$.
To simplify notation, we consider the undiscounted setting without loss of generality.
Our method naturally applies to the discounted setting as long as the horizon is fixed and finite \citep{puterman2014markov}.

A policy $\pi:\mathcal{S}\to\Delta(\fA)$ maps each state to a probability distribution over actions.
We consider the parameterized policies $\pi_\theta$, where the parameters $\theta\in \Theta$ is a vector with $\Theta\subseteq \mathbb{R}^n$ for some constant $n$. Likewise, we parameterize the transition function $p_{\omega} : \mathcal{S} \times \mathcal{A} \to \Delta(\mathcal{S})$ by a parameter $\omega\in \Omega$, where $\Omega\subseteq \mathbb{R}^m$ and is compact. Unless stated otherwise, all norms used in this work are Euclidean (i.e., $\norm{x}=\norm{x}_2$).

The MDP process begins at time step $0$, where
an initial state $S_0$ is sampled from $p_0$.
At each time step $t \in [T-1]$,
an action $A_t$ is sampled based on $\pi(\cdot \mid S_t)$.
Then, a finite reward $R_{t+1} \doteq r(S_t, A_t)$ is given by the environment and a successor state $S_{t+1}$ is obtained based on $p(\cdot \mid S_t, A_t)$. 
After $T$ steps, the agent's interaction with the environment terminates. If the agent reaches any terminal state before time step $T$, it stays there and receives zero reward.

We use $h\doteq\{S_0,A_0,R_1,S_1,A_1,...S_{T-1}, A_{T-1}, R_T\}$ to denote the trajectory of this MDP. We then define the \textit{return} of $h$ as $g(h)\doteq \sum_{t=0}^{T-1}R_{t+1}$. For any policy, we have a distribution over the trajectory as $\mathrm{Pr}(H=h|\pi)$, where $H$ is a random variable used to denote the trajectory. Lastly, we define the \textit{value} of a policy as $v(\pi)\doteq\E_{H\sim\pi}[g(H)]$. To simplify notations, we also define $\textstyle\ell_{p_\omega}\doteq\sum_{t=0}^{T-1}\log(p_{\omega}(S_{t+1}|S_t,A_t))$.

\subsection{Variance Reduction in Policy Evaluation}
We consider the task of reinforcement learning policy evaluation, where the goal is to estimate the value of an interested policy $\pi_{e}$, called the \textit{target policy}. The traditional \textit{on-policy} Monte Carlo (MC) method estimates $v(\pi_e)$ by repeatedly executing the target policy $\pi_e$ online and averaging the observed returns. That is, $\mathrm{MC}(\pi_e,H)\doteq \frac{1}{n}\sum_{i=0}^{n-1}g(H_i)$ for all $ H_i\sim \pi_e$.
However, in practice, this straightforward method can induce high evaluation variance, leading to less reliable results \citep{liu2024efficient,liu2024efficientmul,liu2024doubly, chen2024efficient}.  

To mitigate this challenge, recent work has proposed to use \textit{off-policy} evaluation method to reduce variance, where we execute a different policy $\pi_\theta$ (the \textit{behavior policy}), to collect data. For wide applicability, we consider a general off-policy estimator $\OPE(\pi_e,\pi_\theta,H)$, which estimates the value of $\pi_e$ using trajectories $H$ from $\pi_\theta$. A standard example is importance sampling, $\IS(\pi_e,\pi_\theta,H)\doteq g(H)\prod_{t=0}^{T-1}\tfrac{\pi_e(A_t|S_t)}{\pi_\theta(A_t|S_t)}$. 
Prior work has shown that with a properly designed behavior policy $\pi_\theta$, one can achieve lower evaluation variance with an off-policy estimator than the traditional on-policy MC method \citep{hanna2017data,zhong2022robust}. This is known as the \textit{Behavior Policy Search} (BPS) problem, where we aim to solve $\min_{\theta\in \Theta} \V_{H\sim \pi_\theta} \left[\OPE(\pi_e, \pi_\theta, H)\right]$.

\subsection{Robust Behavior Policy Search}
\label{sec: adversarial minmax}
Standard behavior policy search methods typically assume fixed transition probabilities, but in practice there are often discrepancies between simulators and real environments. After learning a variance-reducing behavior policy in simulation, practitioners typically deploy it to evaluate the target policy in the real system. However, because \textit{robustness to dynamics shifts} is not considered during the behavior policy search phase, the resulting policy may still induce high variance when collecting real-world data. Consequently, achieving reliable evaluation often requires a large amount of costly real-world samples, motivating a robustness-aware formulation. 

To address this, we formulate the robust behavior policy search problem as a minimax optimization:
\begin{align}
\min_{\theta\in \Theta} \max_{\omega\in \Omega} \V_{H\sim p_{\omega},\pi_\theta} \left[\OPE(\pi_e, \pi_\theta, H)\right], \label{eq: minmax objective}
\end{align}
where the inner maximization identifies worst-case transition perturbations, and the outer minimization seeks a behavior policy that mitigates this adversarial effect. Such min–max formulations are standard in the robust RL literature for modeling adversarial dynamics \citep{katdare2023marginalized, voloshin2021minimax}. 
Following standard practice in robust RL \citep{iyengar2005robust, nilim2005robust, ho2021partial, wang2023policy}, our robustness guarantees are defined with respect to a user-specified transition uncertainty set, enabling practitioners to encode transition uncertainties appropriate to their application.



To further analyze the min-max problem in 
\eqref{eq: minmax objective}, we can write it as the following equivalent problem
\begin{align}
\min_{\theta\in \Theta}\{\Phi(\theta ) \doteq \max_{\omega\in \Omega}\V_{\Hsample}[\OPE(\pi_{e},\pi_\theta,H)]\},\label{eq: Phi objective}
\end{align}
which minimizes the worst-case evaluation variance \citep{jin2020local}. Notably, the function $\Phi$ is not differentiable, and is neither convex nor concave. Thus, we are unable to solve the problem through direct gradient descent on the function $\Phi$, which motivates our double-loop approach (Algorithm~\ref{alg:  double with inner oracle}) in Section~\ref{sec: outer loop} with global convergence guarantee.

\section{Solving the Inner Loop}
\label{sec: inner loop}
In this section, we study the inner loop of the optimization problem \eqref{eq: minmax objective}, which identifies adversarial dynamics that maximize evaluation variance. We derive analytical gradient expressions of the variance with respect to the transition probability, considering both \emph{on-transition} and \emph{off-transition} cases, in analogy to on-policy and off-policy settings. In the on-transition case, the simulator transition can be modified, so trajectories are sampled directly from the evolving $p_\omega$ at each iteration. In the off-transition case, the simulator transition is fixed at $p_{\omega_0}$, and trajectories are collected under this fixed transition probability while reweighted toward the target $p_\omega$. We introduce Algorithm~\ref{alg: inner variance algorithm}, which adapts $p_\omega$ to maximize evaluation variance, serving as the adversarial player against the robust behavior policy $\pi_\theta$ (Section~\ref{sec: outer loop}). We provide theoretical convergence guarantees for this algorithm. To the best of our knowledge, this is the first work to develop adversarial transition-gradient methods for variance objectives in reinforcement learning.
\subsection{On-Transition Gradient of the Variance}
We begin with the \textit{on-transition} case, analogous to the on-policy setting, where the simulator transition can be directly modified to follow the target transition $p_\omega$ at each iteration. Given a fixed behavior policy $\pi_\theta$, we look for the variance-maximizing adversarial transition $p_\omega$. Formally, we need to solve
\begin{align}
\textstyle\max_{\omega\in \Omega} \V_{\Hsample} \left[\OPE(\pi_e, \pi_\theta, H)\right].
\end{align}
In the following theorem, we give a gradient expression of the evaluation variance. Importantly, this analytical form is general and applies to any off-policy evaluation estimator, forming the foundation of our transition-gradient method.
\begin{restatable}[Transition Gradient of the Variance]{theorem}
{reOOlemmaOOtransitionOOOPE}
\label{lemma: on transition OPE}
For a fixed behavior policy $\pi_\theta$, 
\begin{align}
\textstyle\gradomega \V_{H\sim p_{\omega},\pi_\theta}[\OPE(\pi_e, \pi_\theta, H)]=&\textstyle\E_{\Hsample}\!\qty[\!\OPE(\pi_e, \pi_\theta,H)^2\gradomega\ell_{p_\omega}] \\
-&2\E_{\Hsample}\![\OPE(\!\pi_e, \pi_\theta, H\!)\!]\textstyle\E_{\Hsample}\!\qty[\!\OPE(\!\pi_e, \pi_\theta, H\!)\gradomega\ell_{p_\omega}\! ].
\end{align}
\end{restatable}
\vspace{-0.75em}
Its proof is in Appendix~\ref{sec: safe on transition gradient}. This gradient expression is in expectation forms and can be estimated unbiasedly from sampled trajectories without additional structural assumptions on the OPE estimator. Building on Theorem~\ref{lemma: on transition OPE}, we now present the On-transition Variance Gradient  method in Algorithm~\ref{alg: inner variance algorithm}. We instantiate our algorithm with the importance sampling estimator (IS), but our framework is ready to accommodate any off-policy evaluation estimator. 

To discuss the convergence property of Algorithm~\ref{alg: inner variance algorithm}, we impose the standard Robbins-Monro step-size condition $
\textstyle\sum_{i=0}^\infty \alpha_i=\infty \text{ and } \sum_{i=0}^\infty \alpha^2_i<\infty$ \citep{robbins1951stochastic, liu2025ode, mahadevan2026convergence}. We assume the importance sampling ratio $\frac{\pi_e(a|s)}{\pi_\theta(a|s)}$ exists and is bounded above for all $s$, $a$, and $\theta$ \citep{hanna2024data}. Besides, we require the transition $p_\omega$ to be twice-differentiable with respect to $\omega$ with uniformly bounded first- and second-order derivatives. These conditions for $p_\omega$ hold, for example, when it is parameterized by a neural network with smooth activations and a softmax output layer, and are commonly adopted in policy gradient literature (e.g., \citet{hanna2017data, hanna2024data}). Then, we have the following lemma for the convergence of Algorithm~\ref{alg: inner variance algorithm}, whose proof is in Appendix~\ref{append: transition convergence}.
\begin{restatable}[Transition Gradient Convergence]{lemma}{reOOtransitiongOOconvergence}
\label{lemma: transition convergence}
For a fixed behavior policy $\pi_{\theta}$, Algorithm~\ref{alg: inner variance algorithm} converges. That is,
$\V_{H_i\sim p_{\omega_i},\pi_\theta}[\IS(\pi_e, \pi_\theta,H_i)]$ converges to a finite value and \\$\lim_{i\to \infty}\frac{\partial}{\partial \omega}\V_{H_i\sim p_{\omega_i},\pi_\theta}[\IS(\pi_e, \pi_\theta, H_i)]=0$.
\end{restatable}

\begin{algorithm*}[t]
\caption{On-Transition Variance Gradient.}
\label{alg: inner variance algorithm}
\begin{algorithmic}[1]
\STATE {\bfseries Input:} 
an initial transition parameter $\omega_0$, a target policy $\pi_e$, a fixed behavior policy $\pi_{\theta}$, a number of iteration $n$, a batch size $k$, a step-size  $\alpha_i$ for each $i$\\
\STATE {\bfseries Output:} a final adversarial transition parameter $\omega_n$
\STATE \textbf{For all} $i\in 0,...,n-1$ \textbf{do}
\STATE \hspace{4mm} Sample $k$ trajectories $H\sim \pi_\theta, p_{\omega_i}$
\STATE \hspace{4mm} $\begin{aligned}[t]
&\omega_{i+1} = \omega_i
\textstyle+ \frac{\alpha_i}{k}\sum_{j=1}^k\qty(\IS(\pi_e,\pi_\theta,H^j)^2\sum_{t=0}^{T-1}\gradomega\log (p^j_{\omega_i}(S_{t+1}|S_t,A_t))) \\
&\textstyle-\frac{8\alpha_i}{k^2}
\sum_{j=1}^{\frac{k}{2}}\IS(\pi_e, \pi_\theta,H^j)\sum_{j=\frac{k}{2}+1}^k\qty(\IS(\pi_e, \pi_\theta,H^j)\sum_{t=0}^{T-1}\gradomega\log (p^j_{\omega_i}(S_{t+1}|S_t,A_t)) ).
\end{aligned}$
\STATE \textbf{End for}
\STATE \textbf{Return:} 
$\omega_{n}$
\end{algorithmic}
\end{algorithm*}


In practice, although discrepancies often exist between the transition probability in the deployment environment and the original simulator, the simulator typically remains a reasonable approximation. Thus, to ensure the learned adversarial transition remains realistic rather than overly pessimistic, we also offer an optional Kullback–Leibler (KL) divergence penalty that discourages large deviations between 
$p_\omega$ and the initial simulator transition $p_{\omega_0}$ \citep{tang2024robust, zhang2026pessimism, chen2026offline, chen2026fast,chen2026pessimism}. Given a behavior policy $\pi_\theta$, we consider the following inner-loop optimization problem under KL regularization:
\begin{align}
\max_{\omega\in \Omega} &\V_{H\!\sim\! p_{\omega}\!,\!\pi_\theta}[\OPE(\!\pi_e, \pi_\theta, H)]-\eta\KL(\Pr(\!H|p_{\omega}\!)\|\Pr(\!H|p_{\omega_0}\!)\!),
\end{align}
where $\eta>0$ is the regularization coefficient and the KL-divergence term is defined as $\textstyle\KL(\Pr(\!H|p_{\omega})\|\Pr(\!H|p_{\omega_0}))\!\doteq\!\E_{H\sim p_{\omega},\pi_\theta}\qty[
\!\log \frac{\Pr(H|p_{\omega})}{\Pr(H|p_{\omega_0})}\!].$
We provide the gradient expression of this regularized optimization problem in the following theorem.
\begin{restatable}[Transition Gradient of Variance with KL]{theorem}
{reOOlemmaOOtransitionOOKL}
\label{lemma: on transition KL}
For a fixed behavior policy $\pi_\theta$ and a regularization coefficient $\eta>0$,
\begin{align}
&\gradomega\! \V_{H\!\sim\! p_{\omega}\!,\!\pi_\theta}[\OPE(\!\pi_e, \pi_\theta, H\!)]\!\!-\!\!\eta\KL(\Pr(\!H|p_{\omega}\!)\|\Pr(\!H|p_{\omega_0}\!)\!)=\textstyle\E_{\Hsample}\!\qty[\OPE(\pi_e, \pi_\theta,H)^2\gradomega\ell_{p_\omega}]\\
\vspace{-0.75em}
-&2\E_{\Hsample}\![\OPE(\!\pi_e, \pi_\theta, H\!)\!]\textstyle\E_{\Hsample}\!\qty[\!\OPE(\!\pi_e, \pi_\theta, H\!)\gradomega\ell_{p_\omega}\! \!]\textstyle-\eta\E_{H\sim p_\omega,\pi_\theta}\left[\qty(\gradomega\ell_{p_\omega})\qty(1+\ell_{p_\omega}\!-\!\ell_{p_{\omega_0}})\right].
\end{align}
\end{restatable}
\vspace{-0.75em}
Its proof is in Appendix~\ref{sec: on transition gradient KL}. This regularization balances robustness with realism, ensuring that the learned adversary remains close to plausible dynamics.

\subsection{Off-Transition Gradient of the Variance}
In the \textit{off-transition} case, analogous to the off-policy setting, simulator transitions are fixed at $p_{\omega_0}$ and may differ from the target transition $p_\omega$. This situation arises naturally when using black-box simulators that permit data collection but do not allow modifying transition probabilities (e.g., \citet{komorowski2018artificial}). In this case, we introduce a transition importance sampling ratio to reweight the collected trajectories, mirroring the familiar correction used in off-policy evaluation for policies. For a general off-policy estimator OPE, we overload the notation as
\begin{align}
& \textstyle\mathrm{OPE}(\!\pi_e, \pi_\theta, p_\omega, H\!)\!\doteq\! \textstyle   \frac{\prod_{t=0}^{T-1}p_\omega(S_{t+1}|S_t, A_t)}{\prod_{t=0}^{T-1}p_{\omega_0}(S_{t+1}|S_t, A_t)}\OPE(\!\pi_e, \pi_\theta,\!H\!).
\end{align}
We omit the input $p_{\omega_0}$ in $ \textstyle\mathrm{OPE}(\pi_e, \pi_\theta, p_\omega, H)$ to simplify notations.
Similar to Theorem~\ref{lemma: on transition OPE}, we first give an analytical gradient expression of the evaluation variance.

\begin{restatable}[Off-Transition Gradient of Variance]{theorem}{reOOlemmaOOoffpolicyOOtransitionOOgradient}
\label{lemma: off gradient expression}
When $p_\omega\neq p_{\omega_0}$, for a fixed behavior policy $\pi_\theta$,
\begin{align}
&\gradomega \V_{H\sim p_{\omega_0},\pi_\theta}[\mathrm{OPE}(\pi_e, \pi_\theta,p_\omega,H)]=\textstyle2\E_{H\sim p_{\omega_0},\pi_\theta}\qty[\mathrm{OPE}^2(\pi_e, \pi_\theta,p_\omega,H)\gradomega\ell_{p_\omega}  ]\\
&\textstyle-2\E_{H\sim p_{\omega_0},\pi_\theta}\qty[\mathrm{OPE}(\pi_e, \pi_\theta,p_\omega,H)]\textstyle\E_{H\sim p_{\omega_0},\pi_\theta}\qty[\mathrm{OPE}(\pi_e, \pi_\theta,p_\omega,H)\gradomega\ell_{p_\omega} ].
\end{align}
\end{restatable}
\vspace{-0.5em}
Its proof is in Appendix \ref{sec: safe off transition gradient}. 
In the next theorem, we incorporate a KL-divergence term to penalize large deviations of $p_\omega$ from the simulator's transition $p_{\omega_0}$, with the KL direction chosen so that the expectation aligns with the available sampling distribution $p_{\omega_0}$. This design ensures the realism of the learned adversarial transition. 


\begin{restatable}[Off-transition Gradient of Variance with KL]{theorem}{reOOKL}
\label{lemma: KL}
For a fixed behavior policy $\pi_\theta$ and a regularization coefficient $\eta>0$,
\vspace{-0.5em}
\begin{align}
&\textstyle\gradomega \V_{H\sim p_{\omega_0},\pi_\theta}[\mathrm{OPE}(\pi_e, \pi_\theta,p_\omega,H)]-\eta \KL(\Pr(H|p_{\omega_0})\|\Pr(H|p_{\omega}))\\
=&\textstyle2\E_{H\sim p_{\omega_0},\pi_\theta}\qty[\mathrm{OPE}^2(\pi_e, \pi_\theta,p_\omega,H)\gradomega\ell_{p_\omega}  ]\textstyle-2\E_{H\sim p_{\omega_0},\pi_\theta}\qty[\mathrm{OPE}(\pi_e, \pi_\theta,p_\omega,H)]\\
&\cdot\textstyle\E_{H\sim p_{\omega_0},\pi_\theta}\qty[\mathrm{OPE}(\pi_e, \pi_\theta,p_\omega,H)\gradomega\ell_{p_\omega} ]-\textstyle\eta \E_{H\sim p_{\omega_0},\pi_\theta}\qty[-\gradomega\ell_{p_\omega} ].
\end{align}
\end{restatable}
\vspace{-0.5em}
Its proof is in Appendix \ref{sec: safe off transition gradient KL}. Note that the gradient expression 
in the off-transition setting (Theorem~\ref{lemma: KL}) differs from that in the on-transition case (Theorem~\ref{lemma: on transition KL}), reflecting the distinct data sampling mechanisms.
\vspace{-0.5em}

\section{Solving the Outer Loop}
\label{sec: outer loop}
\vspace{-0.5em}
In this section, we propose a behavior policy search (BPS) method that is robust to potential discrepancies in the environment. Specifically, we adopt a policy gradient approach to search for a variance-reducing behavior policy under an adversarial transition probability. We first introduce this algorithm, theoretically analyzing its global convergence guarantee. Then, in Section~\ref{sec: empirical}, we demonstrate its empirical robustness under perturbed transition probabilities.
\vspace{-1em}
\subsection{Double-Loop Robust Variance Gradient}
To begin with, recall that in \eqref{eq: minmax objective}, our goal is to solve the min-max objective
\vspace{-0.5em}
\begin{align}
\min_{\theta\in \Theta} \max_{\omega\in \Omega} \V_{H\sim p_{\omega},\pi_\theta} \left[\OPE(\pi_e, \pi_\theta, H)\right].
\end{align}
In Section~\ref{sec: inner loop} and Algorithm~\ref{alg: inner variance algorithm}, we present methods to solve the inner maximization problem by performing gradient ascent on the transition parameter $\omega$. In this section, we focus on performing gradient descent for the variance objective on the policy parameter $\theta$. This is also known as the \textit{behavior policy search} problem in off-policy evaluation (OPE) community \citep{hanna2017data, hanna2024data}, which aims at finding a variance minimizing behavior policy to collect data through gradient based methods.
In Lemma~\ref{lemma: variance gradient expression}, we present the gradient expression for variance with respect to the behavior policy adopted from \citet{hanna2017data}.
\begin{restatable}[Variance Gradient Expression]{lemma}{reOOlemmaOOvariancegradient}
\label{lemma: variance gradient expression}
With a fixed transition probability $p_\omega$, $\forall \theta$,
\begin{align}
&\textstyle\frac{\partial}{\partial\theta}\V_{\Hsample}[\IS(\pi_{e},\pi_\theta,H)]=\E_{\Hsample}[-\IS(\pi_{e},\pi_\theta, H)^2\sum_{t=0}^{T-1}\frac{\partial}{\partial\theta}\log \pi_{\theta}(A_t|S_t)].
\end{align}
\end{restatable}
\vspace{-0.5em}
Importantly, this lemma shows that we can estimate the gradient with trajectories sampled from the behavior policy $\pi_\theta$.
With this analytical expression, we are now ready to present our double loop algorithm, named \textit{Double-Loop Robust Variance Gradient} (DRVG).

\begin{algorithm}[ht]
\caption{Double-Loop Robust Variance Gradient (DRVG)}
\label{alg:  double with inner oracle}
\begin{algorithmic}[1]
\STATE {\bfseries Input:} 
 a target policy parameter $\theta_e$, a number of iteration $n$, a batch size $k$, a step-size  $\alpha$,  tolerance sequence $\{\epsilon_i\}$
\STATE {\bfseries Output:} a final robust behavior policy parameter $\theta^*$
\STATE \textbf{For all} $i= 0,...,n-1$ \textbf{do}
\STATE \hspace{4mm}$\begin{aligned}[t]
\text{Find } p_{\omega_i} \text{ s.t. }\!
&\V_{H\sim p_{\omega_i}, \pi_{\theta_i}}
\!\left[\mathrm{IS}(\pi_e,\pi_{\theta_i},H)\right]\ge&
\max_{p_\omega}
 \V_{H\sim p_{\omega}, \pi_{\theta_i}}[\IS(\pi_{e},\pi_{\theta_i},H)]-\epsilon_i.
\end{aligned}$

\STATE 
\hspace{4mm} $\mathcal{G}_i=\frac{\partial}{\partial\theta}\V_{H\sim p_{\omega_i}, \pi_{\theta_i}}[\IS(\pi_{e},\pi_{\theta_i},H)]; \quad \theta_{i+1} =\mathrm{Proj}_{\Theta}[\theta_i-\alpha \mathcal{G}_i]$
\STATE \textbf{End for}
\STATE \textbf{Return:} 
$\bar \theta\doteq \frac{1}{n}\sum_{i=0}^{n-1}\theta_i.$
\end{algorithmic}
\end{algorithm}

The double loop algorithm DRVG iteratively takes gradient steps on the evaluation variance objective to solve the min-max problem in \eqref{eq: minmax objective}. Specifically, the inner loop of DRVG returns a worst-case transition probability $p_{\omega_i}$ up to a precision $
\epsilon_i$, which can be obtained through Algorithm~\ref{alg: inner variance algorithm}. Such a sequence $\{\epsilon_i\}$ introduces more flexibility to this double-loop algorithm, allowing for quick policy updates without hurting the global convergence property. This choice is also adopted by some prior work in the robust MDP community \citep{ho2021partial, wang2023policy}.

In the outer loop, DRVG takes a \textit{projected gradient step} to minimize evaluation variance within the feasible parameter set $\Theta$. A well-known proximal representation of projected gradient in \citet{bertsekas1995nonlinear} is $\theta_{i+1}\in\argmin_{\theta\in\Theta}\langle \mathcal{G}_i ,\theta-\theta_i\rangle+\frac{1}{2\alpha_i}\norm{\theta-\theta_i}^2
=\mathrm{Proj}_{\Theta}[\theta_i-\alpha \mathcal{G}_i],$
where $\mathrm{Proj}_{\Theta}$ is the projection operator onto $\Theta$. 
In other words, it is identical to taking a plain gradient step, and then using the closest feasible point in Euclidean distance within the feasible set. Notably, when the feasible set $\Theta$ is convex, this projected gradient step can be implemented by a convex optimization solver with a quadratic objective \citep{wang2023policy}. Together, this double-loop algorithm yields a robust behavior policy for off-policy evaluation under environment uncertainty.

\subsection{Global Convergence Analysis}
In this subsection, we present the global convergence analysis for Algorithm~\ref{alg:  double with inner oracle}. For the widely-studied policy gradient methods in reinforcement learning \textit{policy improvement}, the objective function is the \textit{performance} of a given target policy.
In our \textit{policy evaluation} setting, however, in order to minimize the ultimate online samples needed in the real-world evaluation, the objective function is the \textit{performance's variance}, 
\begin{align}
\!&\!\V_{H\!\sim p_\omega,\pi_\theta}[\OPE(\pi_e, \pi_\theta, H)]\!\!=\E_{\Hsample}[\OPE(\pi_e, \pi_\theta,H)^2]\!\!-\!\!\E_{\Hsample}[\OPE(\pi_e, \pi_\theta,H)\!]^2.
\end{align}
The non-linear nature of this variance objective introduces additional difficulties, making the min-max optimization problem \eqref{eq: minmax objective} nonconvex-nonconcave, which is widely known to be challenging \citep{jin2020local, nouiehed2019solving, lin2020on}. Besides, the objective function $\Phi(\theta)$ in the equivalent expression \eqref{eq: Phi objective} is generally non-differentiable and nonconvex, making the theoretical analysis to our Algorithm~\ref{alg:  double with inner oracle} even more challenging. In fact, even without the inner minimization problem, finding the global optima of such nonconvex objectives is already NP-hard in the worst case \citep{jin2020local}.

In \textit{policy improvement} regime without robustness consideration (i.e., a single-loop performance maximization problem), recent work has shown that some algorithms are guaranteed to converge to a globally-optimal policy with a non-convex objective function in \textit{tabular }MDPs \citep{agarwal2021theory, bhanadari2021linear}. When robustness is introduced via a min–max formalization, only recently was the first generic algorithm with global convergence proposed \citep{wang2023policy}. However, since their inner maximization objective (policy performance) reduces to a linear program in each update, the setting is considerably simpler than our variance-based objective.

In Section~\ref{sec: empirical}, we demonstrate the empirical performance of our Algorithm~\ref{alg:  double with inner oracle} under \textit{a neural network policy parameterization}. While in this section, for the theoretical analysis of Algorithm~\ref{alg:  double with inner oracle}, we adopt a linear-softmax parameterization for the behavior policy $\pi_\theta$, $\pi_\theta(a|s)\doteq\frac{\exp(\theta^\top_a\phi(s))}{\sum_{a'\in \mathcal{A}}\exp (\theta^\top_{a'}\phi(s))}, $
where $\phi:s\to \mathbb{R}^d$ is a state feature function, and $\theta_a\in \mathbb{R}^d$ is the parameter associated with action $a\in \mathcal{A}$. In this section, we assume that the parameters' feasible set $\Theta$ to be closed and convex with a diameter $D$ (i.e., $\forall \theta, \theta'\in \Theta$, $\norm{\theta-\theta'}\leq D$), and assume the linear feature to be bounded (i.e., $\forall s,$ $\norm{\phi(s)}\leq B$ for $B\in \mathbb{R}$) . This choice enables generalization across states through shared features, 
and makes the variance objective convex in $\theta$. This assumption has been widely adopted in recent theoretical work on policy gradient \citep{agarwal2021theory, yuan2022linear, cayci2024convergence}.

With the smoothness of this linear-softmax parameterization, 
we first establish Lemma~\ref{lemma: lipschitz and smoothness of V}, which characterizes the behavior of the objective function $\V$ with respect to the policy parameter 
$\theta$. This lemma then helps to derive the Lipschitz continuity and convexity of the otherwise non-convex and non-differentiable objective function $\Phi$ in \eqref{eq: Phi objective}.

\begin{restatable}{lemma}
{reOOlemmaOOLipschitzOOSmoothnessOOV}
\label{lemma: lipschitz and smoothness of V}
Under linear-softmax policy parameterization, the objective function $\V_{\Hsample}[\IS(\pi_{e},\pi_\theta,H)]$ is $L_\Theta$-Lipschitz, $\ell_\Theta$-smooth, and convex in $\theta$ with $L_\Theta= \sqrt{2}BC^{2T}T^3$ and $\ell_\Theta= B^2C^{2T}T^3 \qty(5+8T)$, where $C$ denotes an upper bound on the importance sampling ratio with
$\frac{\pi_e(a|s)}{\pi_\theta(a|s)} \leq C,\forall (s,a), \forall \theta \in \Theta.$
\end{restatable}
Its proof is in Appendix~\ref{appendix: lipschitz and smoothness of V}. We assume bounded importance-sampling ratios, as is standard in off-policy evaluation to ensure finite variance \citep{hanna2017data, hanna2024data}, which can be simply satisfied by bounding the behavior policy away from zero.
 While the theoretical constants scale with horizon $T$, this dependence is intrinsic to importance sampling–based approaches and has also appeared in prior OPE analyses \citep{liu2018breaking,liu2020understanding}. Our result shows that global convergence still holds with finite-sample guarantees despite this scaling. With Lemma~\ref{lemma: lipschitz and smoothness of V}, we further obtain the desired properties of $\Phi$, which serve as key building blocks for the global convergence of Algorithm~\ref{alg:  double with inner oracle}.
\begin{restatable}{lemma}
{reOOlemmaOOpropertiesOOPhi}
\label{lemma: properties Phi}
The function 
$\Phi(\theta)$ \eqref{eq: Phi objective} is $L_\Theta$-Lipschitz and convex in $\theta$.    
\end{restatable}
Its proof is in Appendix~\ref{appendix: properties Phi}. Equipped with Lemma~\ref{lemma: lipschitz and smoothness of V} and Lemma~\ref{lemma: properties Phi}, we are now ready to establish the convergence analysis despite the inherent nondifferentiability. The following theorem provides a finite-sample convergence guarantee for our double-loop algorithm.

\begin{theorem}[Double loop global convergence]
\label{theorem: double loop convergence}
With a constant step size $\alpha\doteq \frac{D}{L_\Theta\sqrt{n}}$, we have
\begin{align}
\textstyle\Phi(\bar \theta)-\min_{\theta\in \Theta}\Phi(\theta)\leq \frac{DL_\Theta}{\sqrt{n}}+\frac{1}{n}\sum_{i=0}^{n-1}\epsilon_i.
\end{align}
\end{theorem}
Its proof is in Appendix~\ref{appendix: proof of double loop convergence}. This result shows that Algorithm~\ref{alg:  double with inner oracle} converges to an $\epsilon-$optimal solution at a rate of $\mathcal{O}(\frac{1}{\sqrt{n}})$, where $n$ is the number of iterations. This rate matches the optimal rate of projected gradient descent in convex optimization, although our min–max variance objective is more challenging than the performance-based objectives studied in prior work \citep{agarwal2021theory, bhanadari2021linear, wang2023policy}. The error bound consists of two parts: the first term $\frac{DL_\Theta}{\sqrt{n}}$ reflects the convergence rate of projected gradient descent, while the second term $\frac{1}{n}\sum_{i=0}^{n-1}\epsilon_i$ accounts for the chosen precision in the inner maximization. 
To our knowledge, this is the \textit{first} global convergence guarantee for variance-minimizing behavior policy search under adversarial transitions, filling an important gap between classical off-policy evaluation and robust reinforcement learning. 
Finally, we note that double-loop adversarial optimization is standard in robust reinforcement learning (e.g., \citet{wang2023policy, ho2021partial, wang2024policy}). As in prior work, our method trades additional, low-cost simulator computation for improved robustness and reliability under transition uncertainty.




\begin{figure*}[t]
\includegraphics[width=1\textwidth]{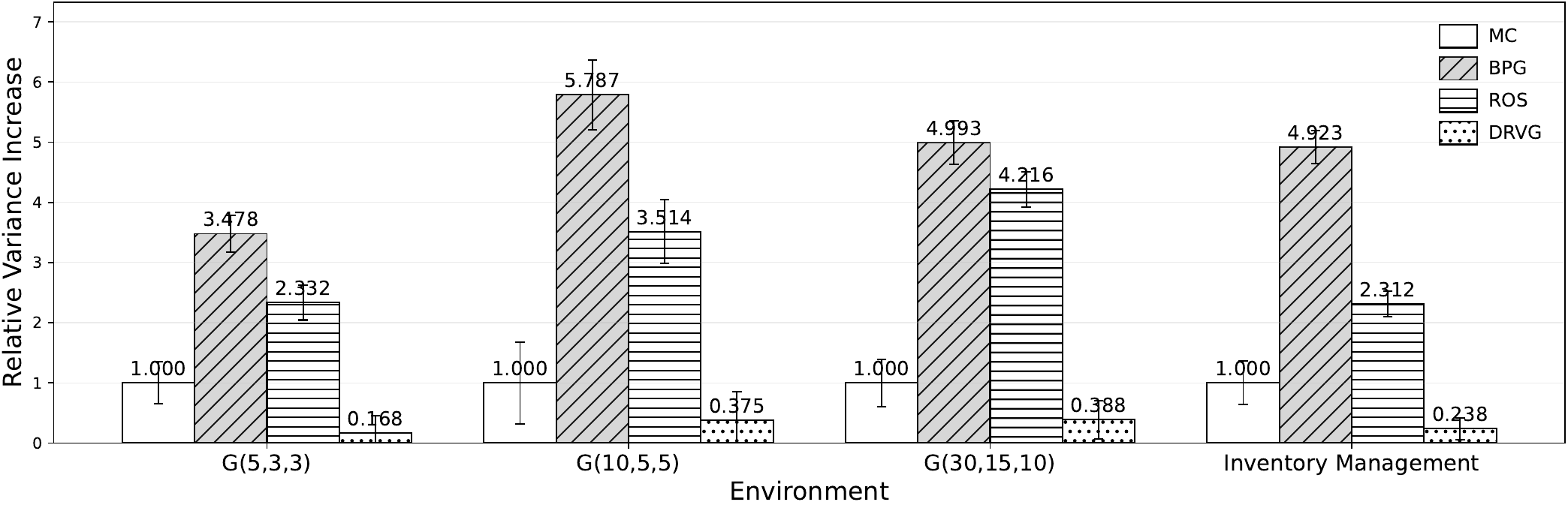}
\centering
\vspace{-1em}
\caption{Relative variance increase of each method under its tailored adversarial transition, compared to the variance under the original simulator transition. All values are normalized by the variance increase of the on-policy Monte Carlo (MC) method in the same environment. More details are provided in Appendix~\ref{appendix: experiments}. Error bars denote the standard error.
}
\label{fig:variance increase}
\vspace{-1em}
\end{figure*}
\vspace{-0.5em}
\section{Numerical Results}
\label{sec: empirical}
\vspace{-0.5em}
In this section, we provide numerical results to validate the utility of our efficient and robust evaluation framework. 
Our primary goal is to examine two key questions: \textbf{(1)} Is our method robust to adversarial transition perturbations? \textbf{(2)} Does it give lower evaluation variance under perturbed transitions compared with standard on-policy Monte Carlo? 

We evaluate these questions on two environments. Garnet MDPs \citep{archibald1995generations} provide a class of randomly generated abstract MDPs that allow controlled investigation of robustness properties. A Garnet instance $G(|\mathcal{S}|, |\mathcal{A}|, b)$ is parameterized by the number of states $|\mathcal{S}|$, number of actions $|\mathcal{A}|$, and a branching factor $b$ that controls the connectivity of transitions. Owing to this flexibility, Garnets are a standard setting for analyzing robustness in controlled MDP studies \citep{tarbouriech2019active, wang2023policy, wang2023robust}.
Inventory management \citep{porteus2002foundations,ho2018fast} is a classical stochastic control problem where a retailer makes ordering decisions under uncertain demand. It provides a natural testbed for evaluating policy performance under transition uncertainties. 
Notably, compared with recent related work in robust reinforcement learning, our experimental environments operate at comparable or higher complexity (see Table~\ref{tab:robust_rl_envs}). 
\vspace{-0.75em}
\begin{table}[h]
\centering
\footnotesize
\begin{minipage}[t]{0.46\linewidth}
\vspace{0pt}
\centering
\setlength{\tabcolsep}{4pt}
\renewcommand{\arraystretch}{0.9}
\begin{tabular}{@{}lcc@{}}
\toprule
\textbf{Method} & \textbf{Garnet} & \textbf{Inventory} \\
\midrule
\textbf{Ours} & $\mathbf{G(30,15)}$ & \cmark \\
\citep{wang2023policy} ICML'23 & $G(15,8)$ & \cmark \\
\citep{sun2024policy} NeurIPS'24 & $G(10,5)$ & \xmark \\
\bottomrule
\end{tabular}
\end{minipage}
\hfill
\begin{minipage}[t]{0.48\linewidth}
\vspace{0pt}
\caption{Environments used in recent robust-RL works. Larger $G(\cdot,\cdot)$ settings indicate more challenging Garnet tasks. Branching factors are omitted as they are not specified in the related work.}
\label{tab:robust_rl_envs}
\end{minipage}
\vspace{-0.75em}
\end{table}


\begin{figure*}[t]
\includegraphics[width=1\textwidth]{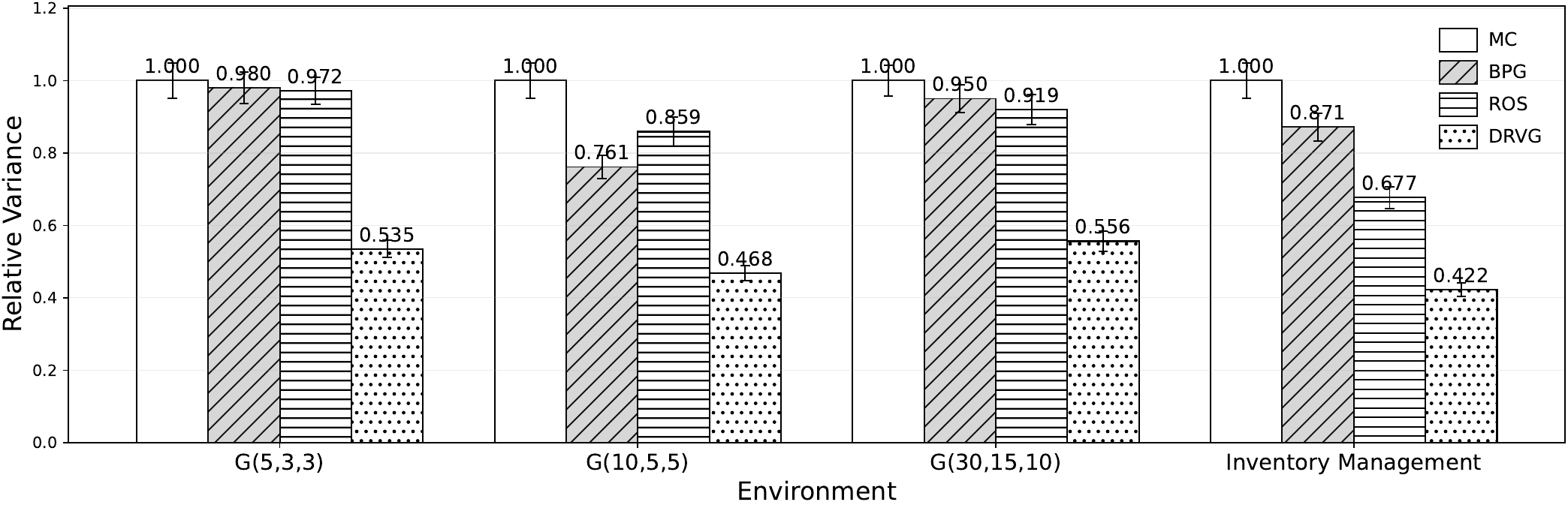}
\centering
\vspace{-0.75em}
\caption{ 
Relative variance of each method under the same perturbed transition. Values are normalized by the variance of the on-policy Monte Carlo (MC) method in the same environment. Error bars denote the standard error.
}
\label{fig:variance compare}
\vspace{-1.25em}
\end{figure*} 

To contextualize the results, we compare our approach with several representative methods: the on-policy Monte Carlo estimator (MC), the behavior policy gradient estimator (BPG, \citet{hanna2017data}), and the robust on-policy sampling estimator (ROS, \citet{zhong2022robust}). All methods are trained with the same initial transition function to obtain their behavior policies. We parameterize our behavior policy with a neural network and use the final iterate behavior policy from Algorithm~\ref{alg: double with inner oracle} to collect evaluation data. Further experimental details are provided in Appendix~\ref{appendix: experiments}.

We note that several related works consider settings that are not directly comparable to our online behavior-policy search framework. In particular, \citet{liu2024efficient} study a fully offline setting with fixed transitions, where the behavior policy is computed from pre-logged data and cannot adapt to transition shifts at deployment. Other robust evaluation approaches \citep{katdare2023marginalized, voloshin2021minimax} focus on estimator robustness or model learning under transition uncertainty, rather than learning variance-minimizing behavior policies for data collection. We therefore include baselines (BPG and ROS) that explicitly target the same behavior-policy optimization objective as our method.

\subsection{Variance Increase under Tailored Adversarial Transitions}
\label{sec: tailored adversary}
To answer the first question, we examine the robustness of each behavior policy when exposed to its own most adversarial transition. For each method, we run Algorithm~\ref{alg: inner variance algorithm} to obtain the transition that maximizes its evaluation variance. 
Then, we use each behavior policy to collect data under its \textit{method-specific adversarial transition}. 
We report the relative variance increase compared to the original simulator transition, highlighting each method’s vulnerability to adversarial perturbations. As shown in Figure~\ref{fig:variance increase}, our method  exhibits the smallest variance increase, illustrating its robustness to adversarial transitions. Notably, although designed for variance reduction, BPG and ROS incur larger variance increases than the on-policy Monte Carlo baseline when the deployment transition is perturbed, underscoring the necessity of our robustness-aware behavior policy search framework.

\subsection{Variance Comparison under Shared and Perturbed Transition}
\label{sec: shared perturbed transition}
To address the second question, 
we evaluate all methods under a shared adversarial target transition identified by Algorithm~\ref{alg: inner variance algorithm} for the on-policy baseline. 
We compare the variance of all four methods under this \textit{same perturbed transition}. This setup contrasts with Section~\ref{sec: tailored adversary}, where each method faced its own tailored adversary. As shown in Figure~\ref{fig:variance compare}, our method (DRVG) indeed achieves lower evaluation variance. This demonstrates that explicitly accounting for transition uncertainty enables more reliable policy evaluation under perturbed environments.


\section{Conclusion}
In this work, we present an efficient and robust behavior policy search framework that tackles two central challenges in real-world policy evaluation: variance reduction and transition mismatch. Our method learns variance-reducing behavior policies while explicitly accounting for transition uncertainty through a minimax formulation over adversarial dynamics. Theoretically, we derive novel transition-variance gradient expressions, establish convergence guarantees for the adversarial inner loop, and prove global convergence of our proposed double-loop algorithm. Numerically, our method demonstrates increased robustness under transition perturbations. Taken together, these results unify variance reduction with robustness to transition shifts, offering a promising step toward reliable policy evaluation under uncertainty.

\begin{ack}
Shangtong Zhang acknowledges funding support from the US National Science Foundation under awards III-2128019, SLES-2331904, and CAREER-2442098, the Commonwealth Cyber Initiative's Central Virginia Node under award VV-1Q26-001, and a Cisco Faculty Research Award.
Nan Jiang acknowledges funding support from NSF CNS-2112471, NSF CAREER IIS-2141781,
and Sloan Fellowship.
\end{ack}




\bibliographystyle{apalike}
\bibliography{bibliography}

@misc{liu2026ortransformer,
  title  = {{OR-Transformer}: Scaling Real-Time Decision-Making to 1,000 Items},
  author = {Liu, Shuze Daniel and Chen, Claire and Gao, Chutong and Simchi-Levi, David},
  year   = {2026},
  note   = {Manuscript}
}

@misc{liu2026pessimistic,
  title  = {Pessimistic Minimax Learning for Public-Private Information Games under Unilateral Coverage},
  author = {Liu, Shuze Daniel and Chen, Claire and Wang, Jiuqi and Simchi-Levi, David},
  year   = {2026},
  note   = {Manuscript}
}

@article{wang2026predicting,
  title={Predicting Plasticity in Deep Continual Learning: A Theoretical Perspective},
  author={Wang, Jiuqi and Srinivasa, Jayanth and Chen, Claire and Liu, Shuze Daniel and Payani, Ali and Zhang, Shangtong},
  journal={arXiv preprint arXiv:2605.09044},
  year={2026}
}

@article{liu2026mathliblemma,
  title={Mathliblemma: Folklore lemma generation and benchmark for formal mathematics},
  author={Liu, Xinyu and Xie, Zixuan and Moeini, Amir and Chen, Claire and Liu, Shuze Daniel and Meng, Yu and Zhang, Aidong and Zhang, Shangtong},
  journal={arXiv preprint arXiv:2602.02561},
  year={2026}
}

@article{zhang2026pessimism,
      title={Beyond Pessimism: Offline Learning in KL-regularized Games}, 
      author={Yuheng Zhang and Claire Chen and Nan Jiang},
        journal={arXiv preprint arXiv:2604.06738},
  year={2026}
}

@inproceedings{
chen2026offline,
title={Offline Two-Player Zero-Sum Markov Games with {KL} Regularization},
author={Claire Chen and Yuheng Zhang and Xinyu Liu and Zixuan Xie and Shuze Daniel Liu and Nan Jiang},
booktitle={Forty-third International Conference on Machine Learning},
year={2026},
url={https://openreview.net/forum?id=cQD2wxFhFG}
}

@article{chen2026pessimism,
  title={Pessimism-Free Offline Learning in General-Sum Games via {KL} Regularization},
  author={Chen, Claire and Zhang, Yuheng},
  journal={ArXiv Preprint arXiv:2605.00264},
  year={2026}
}

@article{chen2026fast,
  title={Fast Rates in $\alpha$-Potential Games via Regularized Mirror Descent},
  author={Chen, Claire and Zhang, Yuheng},
  journal={ArXiv Preprint arXiv:2605.00268},
  year={2026}
}

@inproceedings{tang2024robust,
  title = 	 {Robust Offline Reinforcement Learning with Linearly Structured $f$-Divergence Regularization},
  author =       {Tang, Cheng and Liu, Zhishuai and Xu, Pan},
  booktitle = 	 {Proceedings of the 42nd International Conference on Machine Learning},
  pages = 	 {58842--58882},
  year = 	 {2025}
}

@phdthesis{liu2025efficientrobust,
  title     = {Efficient and Robust Policy Evaluation for Reinforcement Learning},
  author    = {Liu, Shuze Daniel},
  school    = {University of Virginia},
  year      = {2025}
}

@article{xie2026beyond,
  title={Beyond Linear Attention: {Softmax} Transformers Implement In-Context Reinforcement Learning},
  author={Xie, Zixuan and Liu, Xinyu and Chen, Claire and Liu, Shuze Daniel and Chandra, Rohan and Zhang, Shangtong},
  journal={arXiv preprint arXiv:2605.07333},
  year={2026}
}

@article{chen2026astroalertbench,
  title={{AstroAlertBench}: Evaluating the Accuracy, Reasoning, and Honesty of Multimodal {LLM}s in Astronomical Classification},
  author={Chen, Claire and Xiao, Jiabao Sean and Liu, Shuze Daniel and Paolino, Facundo Perez and Handley, Luke and Laz, Theophile Jegou du and Nilsson, Ricky and Zou, Alice and Graham, Matthew and Mahabal, Ashish},
  journal={arXiv preprint arXiv:2605.05573},
  year={2026}
}

@article{liu2026strategic,
  title={Strategic Bargaining in Multi-Buyer Markets: Reinforcement Learning from Verifiable Rewards for {LLM} Negotiations},
  author={Liu, Shuze Daniel and Chen, Claire and Xiao, Jiabao Sean and Chen, Xin and Simchi-Levi, David},
  journal={arXiv preprint arXiv:2607.05863},
  year={2026}
}

@inproceedings{mahadevan2026convergence,
  title={Convergence of Two-Timescale {Markovian} Stochastic Approximations with Applications in Reinforcement Learning},
  author={Mahadevan, Vagul and Chen, Claire and Liu, Shuze Daniel and Zhang, Shangtong},
  booktitle={Proceedings of the 43rd International Conference on Machine Learning},
  year={2026}
}

@article{liu2026instructing,
  title={Instructing {LLM}s to Negotiate Using Reinforcement Learning with Verifiable Rewards},
  author={Liu, Shuze Daniel and Chen, Claire and Xiao, Jiabao Sean and Lei, Lei and Zhang, Yuheng and Yue, Yisong and Simchi-Levi, David},
  journal={arXiv preprint arXiv:2604.09855},
  year={2026}
}

@article{russo2025adaptive,
  title={Adaptive exploration for multi-reward multi-policy evaluation},
  author={Russo, Alessio and Pacchiano, Aldo},
  journal={arXiv preprint arXiv:2502.02516},
  year={2025}
}

@article{liu2025ode,
  title={The ODE method for stochastic approximation and reinforcement learning with markovian noise},
  author={Liu, Shuze Daniel and Chen, Shuhang and Zhang, Shangtong},
  journal={Journal of Machine Learning Research},
  volume={26},
  number={24},
  pages={1--76},
  year={2025}
}

@inproceedings{ho2018fast,
  title={Fast Bellman updates for robust MDPs},
  author={Ho, Chin Pang and Petrik, Marek and Wiesemann, Wolfram},
  booktitle={International Conference on Machine Learning},
  pages={1979--1988},
  year={2018},
  organization={PMLR}
}

@book{porteus2002foundations,
  title={Foundations of stochastic inventory theory},
  author={Porteus, Evan L},
  year={2002},
  publisher={Stanford University Press}
}

@inproceedings{tarbouriech2019active,
  title={Active exploration in markov decision processes},
  author={Tarbouriech, Jean and Lazaric, Alessandro},
  booktitle={The 22nd International Conference on Artificial Intelligence and Statistics},
  pages={974--982},
  year={2019},
  organization={PMLR}
}

@inproceedings{wang2023robust,
  title={Robust average-reward markov decision processes},
  author={Wang, Yue and Velasquez, Alvaro and Atia, George and Prater-Bennette, Ashley and Zou, Shaofeng},
  booktitle={Proceedings of the AAAI Conference on Artificial Intelligence},
  pages={15215--15223},
  year={2023}
}

@article{archibald1995generations,
  title={On the generation of markov decision processes},
  author={Archibald, Thomas Welsh and McKinnon, KIM and Thomas, Lyn C},
  journal={Journal of the Operational Research Society},
  volume={46},
  number={3},
  pages={354--361},
  year={1995},
  publisher={Taylor \& Francis}
}

@article{nilim2005robust,
  title={Robust control of Markov decision processes with uncertain transition matrices},
  author={Nilim, Arnab and El Ghaoui, Laurent},
  journal={Operations Research},
  volume={53},
  number={5},
  pages={780--798},
  year={2005},
  publisher={INFORMS}
}

@article{silver2017mastering,
  title={Mastering the game of go without human knowledge},
  author={Silver, David and Schrittwieser, Julian and Simonyan, Karen and Antonoglou, Ioannis and Huang, Aja and Guez, Arthur and Hubert, Thomas and Baker, Lucas and Lai, Matthew and Bolton, Adrian and others},
  journal={nature},
  volume={550},
  number={7676},
  pages={354--359},
  year={2017},
  publisher={Nature Publishing Group UK London}
}

@inproceedings{voloshin2021minimax,
  title={Minimax model learning},
  author={Voloshin, Cameron and Jiang, Nan and Yue, Yisong},
  booktitle={International Conference on Artificial Intelligence and Statistics},
  pages={1612--1620},
  year={2021},
  organization={PMLR}
}

@inproceedings{katdare2023marginalized,
  title={Marginalized importance sampling for off-environment policy evaluation},
  author={Katdare, Pulkit and Jiang, Nan and Driggs-Campbell, Katherine Rose},
  booktitle={Conference on Robot Learning},
  pages={3778--3788},
  year={2023},
  organization={PMLR}
}

@inproceedings{liu2024efficientmul,
  title={Efficient Multi-Policy Evaluation for Reinforcement Learning},
  author={Liu, Shuze and Chen, Yuxin and Zhang, Shangtong},
  booktitle={Proceedings of the AAAI Conference on Artificial Intelligence},
  year={2025}
}

@inproceedings{chen2024efficient,
  title={Efficient Policy Evaluation with Safety Constraint for Reinforcement Learning},
  author={Chen, Claire and Liu, Shuze and Zhang, Shangtong},
  booktitle={Proceedings of the International Conference on Learning Representations},
  year={2025}
}

@article{komorowski2018artificial,
  title={The artificial intelligence clinician learns optimal treatment strategies for sepsis in intensive care},
  author={Komorowski, Matthieu and Celi, Leo A and Badawi, Omar and Gordon, Anthony C and Faisal, A Aldo},
  journal={Nature medicine},
  volume={24},
  number={11},
  pages={1716--1720},
  year={2018},
  publisher={Nature Publishing Group US New York}
}

@inproceedings{liu2020understanding,
  title={Understanding the curse of horizon in off-policy evaluation via conditional importance sampling},
  author={Liu, Yao and Bacon, Pierre-Luc and Brunskill, Emma},
  booktitle={International Conference on Machine Learning},
  pages={6184--6193},
  year={2020},
  organization={PMLR}
}

@article{cayci2024convergence,
  title={Convergence of entropy-regularized natural policy gradient with linear function approximation},
  author={Cayci, Semih and He, Niao and Srikant, Rayadurgam},
  journal={SIAM Journal on Optimization},
  volume={34},
  number={3},
  pages={2729--2755},
  year={2024},
  publisher={SIAM}
}

@article{yuan2022linear,
  title={Linear convergence of natural policy gradient methods with log-linear policies},
  author={Yuan, Rui and Du, Simon S and Gower, Robert M and Lazaric, Alessandro and Xiao, Lin},
  journal={arXiv preprint arXiv:2210.01400},
  year={2022}
}

@InProceedings{bhanadari2021linear,
  title = 	 { On the Linear Convergence of Policy Gradient Methods for Finite MDPs },
  author =       {Bhandari, Jalaj and Russo, Daniel},
  booktitle = 	 {Proceedings of The 24th International Conference on Artificial Intelligence and Statistics},
  pages = 	 {2386--2394},
  year = 	 {2021},
  editor = 	 {Banerjee, Arindam and Fukumizu, Kenji},
  volume = 	 {130},
  series = 	 {Proceedings of Machine Learning Research},
  month = 	 {13--15 Apr},
  publisher =    {PMLR},
  url = 	 {https://proceedings.mlr.press/v130/bhandari21a.html}
}

@article{agarwal2021theory,
  author  = {Alekh Agarwal and Sham M. Kakade and Jason D. Lee and Gaurav Mahajan},
  title   = {On the Theory of Policy Gradient Methods: Optimality, Approximation, and Distribution Shift},
  journal = {Journal of Machine Learning Research},
  year    = {2021},
  volume  = {22},
  number  = {98},
  pages   = {1--76},
  url     = {http://jmlr.org/papers/v22/19-736.html}
}

@InProceedings{lin2020on,
  title = 	 {On Gradient Descent Ascent for Nonconvex-Concave Minimax Problems},
  author =       {Lin, Tianyi and Jin, Chi and Jordan, Michael},
  booktitle = 	 {Proceedings of the 37th International Conference on Machine Learning},
  pages = 	 {6083--6093},
  year = 	 {2020},
  editor = 	 {III, Hal Daumé and Singh, Aarti},
  volume = 	 {119},
  series = 	 {Proceedings of Machine Learning Research},
  month = 	 {13--18 Jul},
  publisher =    {PMLR},
  url = 	 {https://proceedings.mlr.press/v119/lin20a.html}
}

@inproceedings{nouiehed2019solving,
 author = {Nouiehed, Maher and Sanjabi, Maziar and Huang, Tianjian and Lee, Jason D and Razaviyayn, Meisam},
 booktitle = {Advances in Neural Information Processing Systems},
 editor = {H. Wallach and H. Larochelle and A. Beygelzimer and F. d\textquotesingle Alch\'{e}-Buc and E. Fox and R. Garnett},
 pages = {},
 publisher = {Curran Associates, Inc.},
 title = {Solving a Class of Non-Convex Min-Max Games Using Iterative First Order Methods},
 url = {https://proceedings.neurips.cc/paper_files/paper/2019/file/25048eb6a33209cb5a815bff0cf6887c-Paper.pdf},
 volume = {32},
 year = {2019}
}

@article{ho2021partial,
  author  = {Chin Pang Ho and Marek Petrik and Wolfram Wiesemann},
  title   = {Partial Policy Iteration for L1-Robust Markov Decision Processes},
  journal = {Journal of Machine Learning Research},
  year    = {2021},
  volume  = {22},
  number  = {275},
  pages   = {1--46},
  url     = {http://jmlr.org/papers/v22/20-445.html}
}

@inproceedings{jin2020local,
  title={What is local optimality in nonconvex-nonconcave minimax optimization?},
  author={Jin, Chi and Netrapalli, Praneeth and Jordan, Michael},
  booktitle={International conference on machine learning},
  pages={4880--4889},
  year={2020},
  organization={PMLR}
}

@article{sun2024policy,
  title={Policy optimization for robust average reward mdps},
  author={Sun, Zhongchang and He, Sihong and Miao, Fei and Zou, Shaofeng},
  journal={Advances in Neural Information Processing Systems},
  volume={37},
  pages={17348--17372},
  year={2024}
}

@article{iyengar2005robust,
  title={Robust dynamic programming},
  author={Iyengar, Garud N},
  journal={Mathematics of Operations Research},
  volume={30},
  number={2},
  pages={257--280},
  year={2005},
  publisher={INFORMS}
}

@inproceedings{liu2024doubly,
  title={Doubly Optimal Policy Evaluation for Reinforcement Learning},
  author={Liu, Shuze and Chen, Claire and Zhang, Shangtong},
  booktitle={Proceedings of the International Conference on Learning Representations},
  year={2025}
}

@inproceedings{wang2023policy,
  title={Policy gradient in robust mdps with global convergence guarantee},
  author={Wang, Qiuhao and Ho, Chin Pang and Petrik, Marek},
  booktitle={International Conference on Machine Learning},
  pages={35763--35797},
  year={2023},
  organization={PMLR}
}

@article{hanna2024data,
  title={Data-Efficient Policy Evaluation Through Behavior Policy Search},
  author={Hanna, Josiah P and Chandak, Yash and Thomas, Philip S and White, Martha and Stone, Peter and Niekum, Scott},
  journal={Journal of Machine Learning Research},
  volume={25},
  number={313},
  pages={1--58},
  year={2024}
}

@article{wang2024policy,
  title={Policy Gradient for Robust Markov Decision Processes},
  author={Wang, Qiuhao and Xu, Shaohang and Ho, Chin Pang and Petrik, Marek},
  journal={arXiv preprint arXiv:2410.22114},
  year={2024}
}

@article{bertsekas2000gradient,
  title={Gradient convergence in gradient methods with errors},
  author={Bertsekas, Dimitri P and Tsitsiklis, John N},
  journal={SIAM Journal on Optimization},
  volume={10},
  number={3},
  pages={627--642},
  year={2000},
  publisher={SIAM}
}

@article{jumper2021highly,
  title={Highly accurate protein structure prediction with AlphaFold},
  author={Jumper, John and Evans, Richard and Pritzel, Alexander and Green, Tim and Figurnov, Michael and Ronneberger, Olaf and Tunyasuvunakool, Kathryn and Bates, Russ and {\v{Z}}{\'\i}dek, Augustin and Potapenko, Anna and others},
  journal={Nature},
  year={2021},
}

@inproceedings{liu2024efficient,
  author={Liu, Shuze and Zhang, Shangtong},
  booktitle={Proceedings of the International Conference on Machine Learning},
  title={Efficient Policy Evaluation with Offline Data Informed Behavior Policy Design},
  year={2024}
}

@inproceedings{zhong2022robust,
	author = {Zhong, Rujie and Zhang, Duohan and Sch{\"a}fer, Lukas and Albrecht, Stefano V. and Hanna, Josiah P.},
	booktitle = {Advances in Neural Information Processing Systems},
	title = {Robust On-Policy Sampling for Data-Efficient Policy Evaluation in Reinforcement Learning},
	year = {2022}}

@inproceedings{hanna2017data,
	author = {Hanna, Josiah P and Thomas, Philip S and Stone, Peter and Niekum, Scott},
	booktitle = {Proceedings of the International Conference on Machine Learning},
	title = {Data-efficient policy evaluation through behavior policy search},
	year = {2017}}

@book{puterman2014markov,
	author = {Puterman, Martin L},
	publisher = {John Wiley \& Sons},
	title = {Markov decision processes: discrete stochastic dynamic programming},
	year = {2014}}

@book{bertsekas1995nonlinear,
	author = {Bertsekas, D.P.},
	publisher = {Athena Scientific},
	title = {Nonlinear Programming},
	year = {1995}}

@article{robbins1951stochastic,
	author = {Robbins, Herbert and Monro, Sutton},
	journal = {The Annals of Mathematical Statistics},
	title = {A stochastic approximation method},
	year = {1951}}

@book{sutton2018reinforcement,
	author = {Sutton, Richard S and Barto, Andrew G},
	publisher = {MIT press},
	title = {Reinforcement Learning: An Introduction (2nd Edition)},
	year = {2018}}

@inproceedings{kingma2014adam,
	author = {Kingma, Diederik P. and Ba, Jimmy},
	booktitle = {Proceedings of the International Conference on Learning Representations},
	title = {Adam: {A} Method for Stochastic Optimization},
	year = {2015}}

@inproceedings{liu2018breaking,
	author = {Liu, Qiang and Li, Lihong and Tang, Ziyang and Zhou, Dengyong},
	booktitle = {Advances in Neural Information Processing Systems},
	title = {Breaking the Curse of Horizon: Infinite-Horizon Off-Policy Estimation},
	year = {2018}}

@article{mnih2015human,
	author = {Mnih, Volodymyr and Kavukcuoglu, Koray and Silver, David and Rusu, Andrei A. and Veness, Joel and Bellemare, Marc G. and Graves, Alex and Riedmiller, Martin A. and Fidjeland, Andreas and Ostrovski, Georg and Petersen, Stig and Beattie, Charles and Sadik, Amir and Antonoglou, Ioannis and King, Helen and Kumaran, Dharshan and Wierstra, Daan and Legg, Shane and Hassabis, Demis},
	journal = {Nature},
	title = {Human-level control through deep reinforcement learning},
	year = {2015}}


\appendix



\newpage

\section{Proof}
\subsection{Definitions}
\label{appendix: definition}
In this section, we show the standard optimization definitions used in our work. 
Consider an optimization problem 
\begin{align}
\min_{x\in \mathcal{X}}f(x)
\end{align}
where $\mathcal{X}\subseteq \mathbb{R}^d$ is nonempty and closed, and $f:\mathbb{R}^d\to \mathbb{R}$. We have the following definitions for Lipschitz continuity and smoothness.
\begin{definition}[Lipschitz Continuity]
The function $f: \mathcal{X}\to \mathbb{R}$ is $L-$Lipschitz if $\forall x_1,x_2\in \mathcal{X},$ 
\begin{align}
\norm{f(x_1)-f(x_2)}\leq L\norm{x_1-x_2}.
\end{align}
\end{definition}
\begin{definition}[Smoothness]
The function $f: \mathcal{X}\to \mathbb{R}$ is $\ell-$smooth if $\forall x_1,x_2\in \mathcal{X},$ 
\begin{align}
\norm{\nabla f(x_1)-\nabla f(x_2)}\leq \ell \norm{x_1-x_2}.
\end{align}
\end{definition}

Our theoretical results are established within a standard analytical framework consistent with prior work in behavior-policy search, robust reinforcement learning, and minimax learning \citep{hanna2017data, agarwal2021theory, wang2023policy,hanna2024data, liu2026pessimistic}. To ensure the existence of well-defined gradients and the validity of our global convergence analysis, we consider environments and parameterizations that satisfy the following standard regularity conditions:

\paragraph{Smoothness of the Transition Model} 
We consider a class of transition functions $p_\omega(s'|s,a)$ that are twice-differentiable with respect to their parameters $\omega$. This is a standard property naturally satisfied by typical softmax or neural-network parameterizations with smooth activation functions \citep{agarwal2021theory, wang2023policy}.

\paragraph{Compactness of the Parameter Spaces}
Consistent with established results in minimax optimization and projected gradient descent, the transition uncertainty set $\Omega$ and the behavior policy parameter space $\Theta$ are assumed to be compact and convex sets \citep{agarwal2021theory}.

\paragraph{Finite Evaluation Variance}
To ensure the stability of the behavior-policy search, we assume the importance sampling ratios are uniformly bounded by a constant $C$. This is the standard condition for ensuring finite evaluation variance in off-policy evaluation \citep{hanna2017data, hanna2024data}. In practice, this is satisfied by choosing a behavior policy parameter space $\Theta$ that keeps the data-collection policy bounded away from zero.

\paragraph{Bounded Feature Representations}
For the global convergence guarantees established in Section 5.2, we assume that the state features $\phi(s)$ are bounded \citep{agarwal2021theory, bhanadari2021linear}. 


With these definitions in hand, we are now ready to present the proofs.

\subsection{Proof of Theorem~\ref{lemma: on transition OPE}}
\label{sec: safe on transition gradient}
\reOOlemmaOOtransitionOOOPE*
\begin{proof}

To prove Theorem~\ref{lemma: on transition OPE}, we aim at decomposing the term $\mathrm{Pr}(H=h\mid p_\omega)$ into two parts: one that depends on $p_\omega$ and one that does not. By the standard trajectory factorization for a fixed initial-state distribution $p_0$ and behavior policy $\pi_\theta$,
\begin{align}
\mathrm{Pr}(H=h\mid p_\omega)=p_0(S_0)\prod_{t=0}^{T-1}\pi_\theta(A_t|S_t)\prod_{t=0}^{T-1}p_\omega(S_{t+1}|S_t, A_t).
\end{align}
We separate the $\omega$-dependent transition factor from the $\omega$-independent factors by defining
\begin{align}
m_{p_\omega}(h)\doteq \prod_{t=0}^{T-1}p_\omega(S_{t+1}|S_t, A_t)\label{def: m}
\end{align}
and
\begin{align}
p(h)\doteq p_0(S_0)\prod_{t=0}^{T-1}\pi_\theta(A_t|S_t)
\label{def: p(h)}.
\end{align}
Note that $p(h)$ is independent of $\omega$ since $p_0$ is the fixed initial-state distribution and $\pi_\theta$ does not depend on $\omega$. We thus obtain the decomposition
\begin{align}
\mathrm{Pr}(H=h\mid p_\omega)=p(h)m_{p_\omega}(h). \label{eq: PrH decompose}
\end{align}
Next, we manipulate the term $\gradomega m_{p_\omega}(h)$.
\begin{align}
\gradomega m_{p_\omega}(h)=&\gradomega \prod_{t=0}^{T-1}p_{\omega}(S_{t+1}|S_t, A_t)\\
=&  \sum_{t=0}^{T-1} \left( \prod_{i \neq t} p_\omega(S_{i+1}|S_i,A_i)  \frac{\partial p_\omega( S_{t+1}|S_t,A_t)}{\partial \omega} \right)\\
=&  \sum_{t=0}^{T-1} \left(  \frac{\prod_{i=0}^{T-1} p_\omega(S_{i+1}|S_i,A_i)}{p_\omega( S_{t+1}|S_t,A_t)}.  \frac{\partial p_\omega( S_{t+1}|S_t,A_t)}{\partial \omega} \right)\\
=& \prod_{i=0}^{T-1} p_\omega(S_{i+1}|S_i,A_i)\cdot \sum_{t=0}^{T-1} \left(  \frac{1}{p_\omega( S_{t+1}|S_t,A_t)} \frac{\partial p_\omega( S_{t+1}|S_t,A_t)}{\partial \omega} \right)\\
\overset{\text{(a)}}{=}& \prod_{i=0}^{T-1} p_\omega(S_{i+1}|S_i,A_i)\cdot \sum_{t=0}^{T-1} \left(  \frac{1}{p_\omega( S_{t+1}|S_t,A_t)} p_{\omega}(S_{t+1}|S_t,A_t)\frac{\partial \log p_\omega( S_{t+1}|S_t,A_t)}{\partial \omega} \right)\\
=& \prod_{i=0}^{T-1} p_\omega(S_{i+1}|S_i,A_i)\sum_{t=0}^{T-1} \left( \frac{\partial }{\partial \omega} \log p_\omega( S_{t+1}|S_t,A_t)\right)\\
=&m_{p_\omega}(h)\sum_{t=0}^{T-1} \left( \frac{\partial }{\partial \omega} \log p_\omega( S_{t+1}|S_t,A_t)\right)\\
=& m_{p_\omega}(h)\sum_{t=0}^{T-1}  \gradomega\log (p_\omega(S_{t+1}|S_t,A_t)) \label{eq: m gradient}
\end{align}
Here, (a) follows from the fact that 
\begin{align}
&\frac{\partial}{\partial x} \log f(x) = \frac{1}{f(x)} \frac{\partial f(x)}{\partial x}\\
\implies &\frac{\partial f(x)}{\partial x} = f(x) \cdot \frac{\partial \log f(x)}{\partial x}.
\end{align}
Then, we decompose the variance objective
\begin{align}
&\gradomega \V_{H\sim p_\omega,\pi_\theta}[\OPE(\pi_e, \pi_\theta, H)]\\
=&\gradomega \qty(\E_{\Hsample}[\OPE(\pi_e, \pi_\theta,H)^2]-\E_{\Hsample}[\OPE(\pi_e, \pi_\theta,h)]^2)\\ 
=&\gradomega \sum_h\Pr(H=h|p_\omega)\OPE(\pi_e, \pi_\theta,h)^2\\
&-2\E_{\Hsample}[\OPE(\pi_e, \pi_\theta,H)]\gradomega\sum_h\Pr(H=h|p_\omega)\OPE(\pi_e, \pi_\theta,h)\\
=& \sum_hp(h)\OPE(\pi_e, \pi_\theta,h)^2\gradomega m_{p_\omega}(h)\\
&-2\E_{\Hsample}[\OPE(\pi_e, \pi_\theta,H)]\sum_hp(h)\OPE(\pi_e, \pi_\theta,h)\gradomega m_{p_\omega}(h)\explain{By \eqref{eq: PrH decompose}}\\
\end{align}
\begin{align}
=&\sum_hp(h)\OPE(\pi_e, \pi_\theta,h)^2m_{p_\omega}(h)\sum_{t=0}^{T-1} \gradomega \log (p_\omega(S_{t+1}|S_t,A_t)) \\
&-2\E_{\Hsample}[\OPE(\pi_e, \pi_\theta,H)]\sum_hp(h)\OPE(\pi_e, \pi_\theta,h)m_{p_\omega}(h)\sum_{t=0}^{T-1} \gradomega \log (p_\omega(S_{t+1}|S_t,A_t)) \explain{By \eqref{eq: m gradient}}\\
=&\E_{\Hsample}[\OPE(\pi_e, \pi_\theta,H)^2\sum_{t=0}^{T-1}  \gradomega\log (p_\omega(S_{t+1}|S_t,A_t))]\\
&-2\E_{\Hsample}[\OPE(\pi_e, \pi_\theta,H)]\E_{\Hsample}\qty[\OPE(\pi_e, \pi_\theta, H)\sum_{t=0}^{T-1}\gradomega\log (p_\omega(S_{t+1}|S_t,A_t)) ].
\end{align}

\end{proof}


\subsection{Proof of Lemma~\ref{lemma: transition convergence}}
\label{append: transition convergence}
\reOOtransitiongOOconvergence*
\begin{proof}

The proof leverages Proposition $3$ in \citet{bertsekas2000gradient}, for which we have to show that Algorithm~\ref{alg: inner variance algorithm} satisfies the following conditions:
\begin{enumerate}
\item $\V[\IS(\pi_\theta, p_{\omega_i},H_i)]$ is continuously differentiable w.r.t. $\omega$.
\item The gradient of the variance objectives, $\gradomega \V[\IS(\pi_\theta, p_{\omega_i},H_i)]$, is Lipschitz continuous w.r.t. $\omega$.
\item  The variance of the gradient estimate used by Algorithm~\ref{alg: inner variance algorithm} is bounded.
\end{enumerate}

The other conditions of Proposition $3$ in \citet{bertsekas2000gradient} are satisfied because of the unbiasedness of the gradient estimates in Algorithm~\ref{alg: inner variance algorithm}. Additionally, since the gradient objective, as a variance, is bounded below by zero, we can avoid the case of converging to $-\infty$ according to Proposition $3$ \citep{bertsekas2000gradient}.

By assumptions, we have $p_\omega$ is twice-differentiable, and quotient $\frac{w_{\pi_e}}{w_{\pi_\theta}}$ and the estimator $\ISfunction
$ always exist. Therefore, by the gradient expression in Lemma~\ref{lemma: on transition OPE}, we conclude that $\gradomega V_{H\sim p_\omega,\pi_\theta}[\ISfunction]$ is continuously differentiable, verifying condition 1.

Next, we show the Lipschitz continuity of $\gradomega V_{H\sim p_\omega,\pi_\theta}[\ISfunction]$ by verifying the boundedness of its second derivative.

\begin{align}
&\frac{\partial^2}{\partial \omega^2}\V_{H\sim p_\omega,\pi_\theta}[\ISfunction]\\
=&\gradomega \textstyle\E_{\Hsample}\qty[\IS(\pi_e, \pi_\theta,H)^2\sum_{t=0}^{T-1}\log (p_\omega(S_{t+1}|S_t,A_t)) ]\\
&-2\E_{\Hsample}[\IS(\pi_e, \pi_\theta,H)]\textstyle\E_{\Hsample}\qty[\IS(\pi_e, \pi_\theta,H)\sum_{t=0}^{T-1}\log (p_\omega(S_{t+1}|S_t,A_t)) ]\\
=&\textstyle\gradomega \left(\sum_{h}\qty(p(h)m_{p_{\omega}}(h)\IS(\pi_e, \pi_\theta,H)^2\sum_{t=0}^{T-1} \gradomega \log (p_\omega(S_{t+1}|S_t,A_t)))\right.\\
&\textstyle-2\sum_{h}\qty(p(h)m_{p_{\omega}}(h)\IS(\pi_e, \pi_\theta,H))\\
&\textstyle\left.\cdot\sum_{h}\qty(p(h)m_{p_{\omega}}(h)\IS(\pi_e, \pi_\theta,H)\sum_{t=0}^{T-1} \gradomega \log (p_\omega(S_{t+1}|S_t,A_t)))\right)\explain{By Lemma~\ref{lemma: on transition OPE} and \eqref{eq: PrH decompose}}\\
\end{align}
\begin{align}
=&\textstyle\gradomega \left(\sum_{h}\qty(p(h)m_{p_{\omega}}(h)\IS(\pi_e, \pi_\theta,H)^2\gradomega\log m_{p_\omega}(h))\right.\\
&\textstyle\left.-2\sum_{h}\qty(p(h)m_{p_{\omega}}(h)\IS(\pi_e, \pi_\theta,H))\cdot\sum_{h}\qty(p(h)m_{p_{\omega}}(h)\IS(\pi_e, \pi_\theta,H)\gradomega\log m_{p_\omega}(h))\right)\\
=&\textstyle\gradomega \left(\sum_{h}\qty(p(h)m_{p_{\omega}}(h)\IS(\pi_e, \pi_\theta,H)^2\frac{1}{m_{p_\omega}(h)}\gradomega m_{p_\omega}(h))\right.\\
&\textstyle\left.-2\sum_{h}\qty(p(h)m_{p_{\omega}}(h)\IS(\pi_e, \pi_\theta,H))\cdot\sum_{h}\qty(p(h)m_{p_{\omega}}(h)\IS(\pi_e, \pi_\theta,H)\frac{1}{m_{p_\omega}(h)}\gradomega m_{p_\omega}(h))\right)\\
=&\textstyle\gradomega \left(\sum_{h}\qty(p(h)\IS(\pi_e, \pi_\theta,H)^2\gradomega m_{p_\omega}(h))\right.\\
&\textstyle\left.-2\sum_{h}\qty(p(h)m_{p_{\omega}}(h)\IS(\pi_e, \pi_\theta,H))\cdot\sum_{h}\qty(p(h)\IS(\pi_e, \pi_\theta,H)\gradomega m_{p_\omega}(h))\right)\\
=&\textstyle\sum_hp(h)\qty(\underbrace{\IS(\pi_e, \pi_\theta,H)^2}_{(1)}\underbrace{\frac{\partial^2}{\partial \omega^2}m_{p_\omega}(h)}_{(2)})\\
&\textstyle-2\gradomega \qty[\sum_{h}\qty(p(h)m_{p_{\omega}}(h)\IS(\pi_e, \pi_\theta,H))\cdot\sum_{h}\qty(p(h)\IS(\pi_e, \pi_\theta,H)\gradomega m_{p_\omega}(h))].\\
\end{align}

We further decompose the term in the square brackets.

\begin{align}
&\gradomega \qty[\sum_{h}\qty(p(h)m_{p_{\omega}}(h)\IS(\pi_e, \pi_\theta,H))\cdot\sum_{h}\qty(p(h)\IS(\pi_e, \pi_\theta,H)\gradomega m_{p_\omega}(h))]\\
=&\sum_{h}p(h)\gradomega (m_{p_{\omega}}(h)\IS(\pi_e, \pi_\theta,H))\cdot\sum_{h}\qty(p(h)\IS(\pi_e, \pi_\theta,H)\gradomega m_{p_\omega}(h))\\
&+\sum_{h}\qty(p(h)m_{p_{\omega}}(h)\IS(\pi_e, \pi_\theta,H))\cdot \sum_hp(h)\gradomega\qty(\IS(\pi_e, \pi_\theta,H)\gradomega m_{p_\omega}(h))\\
=&\sum_{h}p(h)\qty(\underbrace{\IS(\pi_e, \pi_\theta,H)}_{(3)}\underbrace{\gradomega m_{p_{\omega}}(h)}_{(4)})\cdot\sum_{h}p(h)\qty(\IS(\pi_e, \pi_\theta,H)\gradomega m_{p_\omega}(h))\\
&+\sum_{h}p(h)\qty(\underbrace{m_{p_{\omega}}(h)}_{(5)}\IS(\pi_e, \pi_\theta,H))\cdot \sum_hp(h)\qty(\IS(\pi_e, \pi_\theta,H)\frac{\partial^2}{\partial \omega^2}m_{p_\omega}(h)).\\
\end{align}

Notice that since $p(h)=p_0(S_0)\prod_{t=0}^{T-1}\pi_\theta(A_t|S_t)\leq 1$ by \eqref{def: p(h)}, $p(h)$ is bounded above. We then analyze the boundedness of $\frac{\partial^2}{\partial \omega^2}V_{H\sim p_\omega,\pi_\theta}[\ISfunction]$ through the above $5$ terms.

For (1) and (3), the quotient $\frac{\pi_e(a|s)}{\pi_\theta(a|s)}$ is bounded above by assumption. Besides, since the reward is bounded, so is $g(h)$. Therefore, both (1), $\IS(\pi_e, \pi_\theta,H)^2$ and (3) $\IS(\pi_e, \pi_\theta,H)$ are bounded. 

For (5), it is bounded because $m_{p_\omega}(h)=\prod_{t=0}^{T-1}p_\omega(S_{t+1}|S_t, A_t)\leq 1$. Then, for (4),
\begin{align}
\gradomega  m_{p_\omega}(h)=&\gradomega \prod_{t=0}^{T-1}p_{\omega}(S_{t+1}|S_t, A_t)\\
=&  \sum_{t=0}^{T-1}  \frac{\partial}{\partial \omega} p_\omega( S_{t+1}|S_t,A_t)\frac{\prod_{i=0}^{T-1} p_\omega(S_{i+1}|S_i,A_i)}{p_\omega( S_{t+1}|S_t,A_t)}. 
\end{align}
Here, $ \frac{\partial}{\partial \omega} p_\omega( S_{t+1}|S_t,A_t)$ is bounded by construction and $\frac{\prod_{i=0}^{T-1} p_\omega(S_{i+1}|S_i,A_i)}{p_\omega( S_{t+1}|S_t,A_t)}\leq 1$. Thus, (4) is bounded. Lastly, for (2)
\begin{align}
&\frac{\partial^2}{\partial \omega^2}m_{p_\omega}(h)\\
=&\gradomega  \sum_{t=0}^{T-1}  \frac{\partial}{\partial \omega} p_\omega( S_{t+1}|S_t,A_t)\frac{\prod_{i=0}^{T-1} p_\omega(S_{i+1}|S_i,A_i)}{p_\omega( S_{t+1}|S_t,A_t)}\\
=&\gradomega  \sum_{t=0}^{T-1}  \frac{\partial}{\partial \omega} p_\omega( S_{t+1}|S_t,A_t)\prod_{i\neq t} p_\omega(S_{i+1}|S_i,A_i)\\
=&\sum_{t=0}^{T-1}  \frac{\partial^2}{\partial \omega^2} p_\omega( S_{t+1}|S_t,A_t)\prod_{i\neq t} p_\omega(S_{i+1}|S_i,A_i)+\gradomega p_\omega( S_{t+1}|S_t,A_t)\\
&\cdot\sum_{i\neq t}\gradomega p_\omega( S_{i+1}|S_i,A_i)\prod_{j\neq t,i} p_\omega( S_{j+1}|S_j,A_j),
\end{align}
which is bounded because $p_\omega$ is constructed to be twice differentiable with bounded first and second derivatives.

Therefore, we conclude that the gradient objective $\frac{\partial}{\partial\omega}V_{H\sim p_\omega,\pi_\theta}[\ISfunction]$ is Lipschitz continuous w.r.t. $\omega$, verifying condition 1.

Finally, we show that the variance of the gradient estimate used by Algorithm~\ref{alg: inner variance algorithm} is bounded.
According to Algorithm~\ref{alg: inner variance algorithm}, we use the unbiased estimate as
\begin{align}
\textstyle\frac{\partial}{\partial\omega}V_{H\sim p_\omega,\pi_\theta}[\ISfunction]\approx &\textstyle\underbrace{\IS(\pi_\theta,p_{\omega},H)^2\sum_{t=0}^{T-1}\gradomega\log (p_\omega(S_{t+1}|S_t,A_t))}_{A}\\
&\textstyle-\underbrace{2\IS(\pi_\theta,p_{\omega},H)\IS(\pi_\theta,p_{\omega},H)\sum_{t=0}^{T-1}\gradomega\log (p_\omega(S_{t+1}|S_t,A_t))}_{B}.
\end{align}
Then, the variance of the estimate is decomposed into 
\begin{align}
\V[A]+\V[B]+2\mathrm{Cov}[A,B],
\end{align}
where $\mathrm{Cov}[A,B]\leq \sqrt{\V[A]}\cdot\sqrt{\V[B]}$ by the Cauchy-Schwarz inequality. Thus, it is sufficient to show the boundedness of $\V[A]$ and $\V[B]$. For $\V[A]$,
since the variance of a bounded random variable is bounded, we aim to demonstrate that for any trajectory $h$, the term $\IS(\pi_\theta,p_{\omega},H)^2\sum_{t=0}^{T-1}\log (p_\omega(S_{t+1}|S_t,A_t))$ is bounded. 
\begin{align}
&\IS(\pi_\theta,p_{\omega},H)^2\sum_{t=0}^{T-1}\gradomega\log (p_\omega(S_{t+1}|S_t,A_t))\\
=&\IS(\pi_\theta,p_{\omega},H)^2\sum_{t=0}^{T-1} \gradomega \log (p_\omega(S_{t+1}|S_t,A_t))\\
=&\IS(\pi_\theta,p_{\omega},H)^2\gradomega \log m_{p_\omega}(h)\\
=&\IS(\pi_\theta,p_{\omega},H)^2 \frac{\gradomega m_{p_\omega}(h)}{m_{p_\omega}(h)}. \label{eq: Delta derivation} \\
\end{align}
The boundedness of $\IS(\pi_\theta,p_{\omega},H)^2$ and $\gradomega m_{p_\omega}(h)$ is shown by the argument above for terms (3) and (4). For the boundedness of $\frac{1}{m_{p_\omega}(h)}=\frac{1}{\prod_{t=0}^{T-1}p_\omega(S_{t+1}|S_t, A_t)}$, we invoke the Extreme Value Theorem. Since $p_\omega(s'|s,a)$ is strictly positive (under softmax parameterization) and continuous on the finite state-action space $\mathcal{S}\times \mathcal{A}\times\mathcal{S}$ and the compact parameter space $\Omega$ (compactness of $\Omega$ is assumed in our background), the Extreme Value Theorem ensures that $p_\omega$ attains a strictly positive minimum $c\doteq \min_{s,a,s',\omega}p_\omega(s'|s,a)>0$. Since the trajectory length $T$ is finite, we have $\frac{1}{m_{p_\omega}(h)}\leq (1/c)^T<\infty$ uniformly in $h$ and $\omega$. Thus, we conclude that $\V[A]$ is bounded.

Next, we decompose term $B$ into two parts because of the different samples used to estimate them:
\begin{align}
\underbrace{\IS(\pi_\theta,p_{\omega},H)}_{C}\underbrace{\IS(\pi_\theta,p_{\omega},H)\textstyle\sum_{t=0}^{T-1}\gradomega\log (p_\omega(S_{t+1}|S_t,A_t))}_{D}.
\end{align}
Since in Algorithm~\ref{alg: inner variance algorithm} the cross-product estimator is computed by sample splitting (the first $k/2$ trajectories estimate $C$ and the second $k/2$ trajectories estimate $D$), $C$ and $D$ are independent. Consequently, by the standard variance identity for products of independent random variables,
\begin{align}
\V[B]=\V[CD]=\E[C^2]\V[D]+\E[D]^2\V[C].
\end{align}
We show their boundedness term by term.
\begin{align}
\E[C^2]=\E_{H\sim p_\omega,\pi_\theta}[\IS(\pi_\theta,p_{\omega},H)^2]=\sum_{h}p(h)m_{p_{\omega}}(h)\IS(\pi_e, \pi_\theta,H)^2,\label{eq: bounded C square}
\end{align}
where each term is shown to be bounded above. Next, by Jensen's inequality and the derivation from \eqref{eq: Delta derivation},
\begin{align}
\E[D]^2\leq \E[D^2]=\sum_h p(h)m_{p_{\omega}}(h)\IS(\pi_e, \pi_\theta,H)^2\qty(\frac{\gradomega m_{p_\omega}(h)}{m_{p_\omega}(h)})^2,
\end{align}
where the boundedness of the right-hand side follows from the analysis of \eqref{eq: bounded C square} and \eqref{eq: Delta derivation}, and hence $\E[D]^2$ is bounded as well.

As for the two variance terms, $\V[C]$ and $\V[D]$, we show the boundedness of the random variable $C$ and $D$ for each trajectory $h$, where
$\IS(\pi_\theta,p_{\omega},H)$ is shown to be bounded in term $(3)$ above, and the boundedness of $\sum_{t=0}^{T-1}\gradomega\log (p_\omega(S_{t+1}|S_t,A_t))$ is incorporated in \eqref{eq: Delta derivation}.

Therefore, we conclude that the variance of our estimate is bounded. By far, we show that the three conditions of Proposition 3 in \citet{bertsekas2000gradient} are satisfied, demonstrating the convergence of Algorithm~\ref{alg: inner variance algorithm}.

\end{proof}
\clearpage
\subsection{Proof of Theorem~\ref{lemma: on transition KL}}
\label{sec: on transition gradient KL}
\begin{proof}
\reOOlemmaOOtransitionOOKL*
We begin by manipulating the KL-divergence term.
\begin{align}
\KL(\Pr(H|p_{\omega})\|\Pr(H|p_{\omega_0}))=&\E_{H\sim p_{\omega},\pi_\theta}\qty[\log \frac{\Pr(H|p_{\omega})}{\Pr(H|p_{\omega_0})}]\\
=&\E_{H\sim p_{\omega},\pi_\theta}\qty[\log \frac{m_{p_{\omega}}(H)}{m_{p_{\omega_0}}(H)}]\explain{By \eqref{eq: PrH decompose}}\\
=&\E_{H\sim p_{\omega},\pi_\theta}\qty[\log m_{p_{\omega}}(H)-\log m_{p_{\omega_0}}(H)].
\end{align}
Next, we decompose the following gradient:
\begin{align}
\label{eq: gradient log mp}
& \gradomega \log m_{p_{\omega}}(H)\\
=&\sum_{t=0}^{T-1}\gradomega \log p_{\omega}(S_{t+1}|S_t,A_t)\\
=&\sum_{t=0}^{T-1}\gradomega\log (p_\omega(S_{t+1}|S_t,A_t)). \explain{By definition}\\
\end{align}
Then, we take the gradient of the KL-divergence with respect to $\omega$:
\begin{align}
\label{eq: gradient KL on}
&\gradomega \KL(\Pr(H|p_{\omega})\|\Pr(H|p_{\omega_0}))\\
=&\gradomega \E_{H\sim p_{\omega},\pi_\theta}\qty[\log m_{p_{\omega}}(H)-\log m_{p_{\omega_0}}(H)]\\
=&\gradomega \sum_h \Pr(H=h|p_\omega)\qty[\log m_{p_{\omega}}(h)-\log m_{p_{\omega_0}}(h)]\\
=&\gradomega \sum_h p(h)m_{p_\omega}(h)\qty[\log m_{p_{\omega}}(h)-\log m_{p_{\omega_0}}(h)] \explain{By \eqref{eq: PrH decompose}}\\
=&\sum_h p(h) \qty[\gradomega m_{p_{\omega}}(h)\log m_{p_{\omega}}(h)-\log m_{p_{\omega_0}}(h)\gradomega m_{p_{\omega}}(h)]\\
=& \sum_h p(h) \left[\log m_{p_{\omega}}(h)\gradomega m_{p_{\omega}}(h)+m_{p_{\omega}}(h)\gradomega\log m_{p_{\omega}}(h)\right.\\
&\left.-\log m_{p_{\omega_0}}(h)m_{p_{\omega}}(h)\sum_{t=0}^{T-1}\gradomega\log (p_\omega(S_{t+1}|S_t,A_t))\right] \explain{By \eqref{eq: m gradient}}\\
\end{align}
\begin{align}
=& \sum_h p(h) \left[\log m_{p_{\omega}}(h)m_{p_{\omega}}(h)\sum_{t=0}^{T-1}\gradomega\log (p_\omega(S_{t+1}|S_t,A_t))+m_{p_{\omega}}(h)\sum_{t=0}^{T-1}\gradomega\log (p_\omega(S_{t+1}|S_t,A_t))\right.\\
&\left.-\log m_{p_{\omega_0}}(h)m_{p_{\omega}}(h)\sum_{t=0}^{T-1}\gradomega\log (p_\omega(S_{t+1}|S_t,A_t))\right] \explain{By \eqref{eq: m gradient} \eqref{eq: gradient log mp}}\\
=& \sum_h p(h) m_{p_{\omega}}(h)\sum_{t=0}^{T-1}\gradomega\log (p_\omega(S_{t+1}|S_t,A_t))\qty[\log m_{p_{\omega}}(h)+1-\log m_{p_{\omega_0}}(h)]\\
=& \sum_h\Pr(H=h|p_\omega)\sum_{t=0}^{T-1}\qty[\gradomega\log (p_\omega(S_{t+1}|S_t,A_t))]\qty[\log m_{p_{\omega}}(h)+1-\log m_{p_{\omega_0}}(h)] \explain{By \eqref{eq: PrH decompose}}\\
=&\E_{H\sim p_\omega,\pi_\theta}\qty[\qty(\gradomega\ell_{p_\omega})\qty(1+\ell_{p_\omega}-\ell_{p_{\omega_0}})].\explain{By \eqref{def: m}}\\
\end{align}
Thus, 
\begin{align}
&\gradomega\V_{H\sim p_{\omega},\pi_\theta} \left[\OPE(\pi_e, \pi_\theta, H)\right]-\eta\KL(\Pr(H|p_{\omega})\|\Pr(H|p_{\omega_0}))\\
=&\textstyle\E_{\Hsample}\qty[\OPE(\pi_e, \pi_\theta,H)^2\sum_{t=0}^{T-1}\log (p_\omega(S_{t+1}|S_t,A_t)) ]-2\E_{\Hsample}[\OPE(\pi_e, \pi_\theta, H)]\\
&\cdot\textstyle\E_{\Hsample}\qty[\OPE(\pi_e, \pi_\theta, H)\sum_{t=0}^{T-1}\log (p_\omega(S_{t+1}|S_t,A_t)) ]-\eta\gradomega \KL(\Pr(H|p_{\omega})\|\Pr(H|p_{\omega_0}))\explain{By Lemma~\ref{lemma: on transition OPE}}\\
=&\textstyle\E_{\Hsample}\qty[\OPE(\pi_e, \pi_\theta,H)^2\gradomega\ell_{p_\omega} ]-2\E_{\Hsample}[\OPE(\pi_e, \pi_\theta, H)]\textstyle\E_{\Hsample}\qty[\OPE(\pi_e, \pi_\theta, H)\gradomega\ell_{p_\omega} ]\\
&\textstyle-\eta\E_{H\sim p_\omega,\pi_\theta}\qty[\qty(\gradomega\ell_{p_\omega})\qty(1+\ell_{p_\omega}-\ell_{p_{\omega_0}})].\explain{By \eqref{eq: gradient KL on}}\\
\end{align}
\end{proof}

\subsection{Proof of Theorem~\ref{lemma: off gradient expression}}
\label{sec: safe off transition gradient}
\reOOlemmaOOoffpolicyOOtransitionOOgradient*
\begin{proof}
For simplification, we define $w_{\pi}(h)\doteq\prod_{t=0}^{T-1}\pi(A_t|S_t)$ under trajectory $h$. Then, 
\begin{align} 
&\gradomega \V_{H\sim p_{\omega_0},\pi_\theta}[\mathrm{OPE}(\pi_e,\pi_{\theta}, p_\omega,H)]\\
=&\gradomega \qty(\E_{\Hsampleoff}[\mathrm{OPE}^2(\pi_e,\pi_{\theta}, p_\omega,H)]-\E_{\Hsampleoff}[\mathrm{OPE}(\pi_e,\pi_{\theta}, p_\omega,H)]^2)\\
=&\gradomega \qty(\E_{\Hsampleoff}\qty[\frac{m^2_{p_\omega}(H)}{m^2_{p_{\omega_0}}(H)}\mathrm{OPE}^2(\pi_e,\pi_{\theta}, H)]-\E_{\Hsampleoff}\qty[\frac{m_{p_\omega}(H)}{m_{p_{\omega_0}}(H)}\mathrm{OPE}(\pi_e,\pi_{\theta}, H)]^2) \\
=&\gradomega \sum_{h}\qty(\Pr(H=h|p_{\omega_0})\frac{m^2_{p_\omega}(H)}{m^2_{p_{\omega_0}}(H)}\mathrm{OPE}^2(\pi_e,\pi_{\theta}, H))-2\E_{\Hsampleoff}\qty[\mathrm{OPE}(\pi_e, \pi_\theta,p_\omega,H)]\\
&\gradomega\E_{\Hsampleoff}\qty[\frac{m_{p_\omega}(H)}{m_{p_{\omega_0}}(H)}\mathrm{OPE}(\pi_e,\pi_{\theta}, H)]\\
=&\sum_h \qty(p(h)\frac{1}{m_{p_{\omega_0}}(H)}\mathrm{OPE}^2(\pi_e,\pi_{\theta}, H)\gradomega m^2_{p_\omega}(h))\\
&-2\E_{\Hsampleoff}\qty[\OPE(\pi_e, \pi_\theta,p_\omega,H)]\gradomega\sum_h\qty(p(h)m_{p_{\omega_0}(h)} \frac{m_{p_\omega}(h)}{m_{p_{\omega_0}}(h)}\OPE(\pi_e, \pi_\theta, H) )\explain{By \eqref{eq: PrH decompose}}\\
=&2\sum_h \qty(p(h)\frac{m_{p_\omega}(h)}{m_{p_{\omega_0}}(h)}\mathrm{OPE}^2(\pi_e,\pi_{\theta}, H)\gradomega m_{p_\omega}(h))\\
&-2\E_{\Hsampleoff}\qty[\OPE(\pi_e, \pi_\theta,p_\omega,H)]\sum_h \qty(p(h)\OPE(\pi_e, \pi_\theta, H) \gradomega m_{p_\omega}(h)) \explain{By \eqref{eq: PrH decompose}}\\
=&2\sum_h \qty(p(h)\mathrm{OPE}^2(\pi_e,\pi_{\theta}, H)\frac{m_{p_\omega}(h)}{m_{p_{\omega_0}}(h)}m_{p_\omega}(h)\gradomega\ell_{p_\omega} )\\
&-2\E_{\Hsampleoff}\qty[\OPE(\pi_e, \pi_\theta,p_\omega,H)]\sum_h\qty( p(h)\OPE(\pi_e, \pi_\theta, H)m_{p_\omega}(h)\gradomega\ell_{p_\omega} ) \explain{By \eqref{eq: m gradient}}\\
=&2\sum_h \qty(p(h) m_{p_{\omega_0}}(h) \frac{m^2_{p_\omega}(h)}{m^2_{p_{\omega_0}}(h)}\mathrm{OPE}^2(\pi_e,\pi_{\theta}, H)\gradomega\ell_{p_\omega} )\\
&-2\E_{\Hsampleoff}\qty[\OPE(\pi_e, \pi_\theta,p_\omega,H)]\\
&\cdot\sum_h\qty( p(h)m_{p_{\omega_0}}(h)\frac{m_{p_\omega}(h)}{m_{p_{\omega_0}}(h)}\OPE(\pi_e, \pi_\theta, H)\sum_{t=0}^{T-1}\log (p_\omega(S_{t+1}|S_t,A_t)) ) \explain{By \eqref{eq: m gradient}}\\
=&2\sum_h \qty(\Pr(H=h|p_{\omega_0})\mathrm{OPE}^2(\pi_e,\pi_{\theta}, p_\omega,H)\gradomega\ell_{p_\omega} )\\
&-2\E_{\Hsampleoff}\qty[\OPE(\pi_e, \pi_\theta,p_\omega,H)]\\
\end{align}
\begin{align}
&\cdot\sum_h  \qty(\Pr(H=h|p_{\omega_0})\OPE (\pi_e, \pi_\theta,p_\omega,H) \sum_{t=0}^{T-1}\gradomega\log (p_\omega(S_{t+1}|S_t,A_t))) \explain{By \eqref{eq: PrH decompose}}\\
=&2\E_{\Hsampleoff}\qty[\mathrm{OPE}^2(\pi_e,\pi_{\theta}, p_\omega,H)\gradomega\ell_{p_\omega} ]\\
&-2\E_{\Hsampleoff}\qty[\mathrm{OPE}(\pi_e,\pi_{\theta}, p_\omega,H)]\E_{\Hsampleoff}\qty[\mathrm{OPE}(\pi_e,\pi_{\theta}, p_\omega,H)\gradomega\ell_{p_\omega} ].\\
\end{align}
\end{proof}

\subsection{Proof of Theorem~\ref{lemma: KL}}
\label{sec: safe off transition gradient KL}
\reOOKL*
The KL-divergence between two probability distribution $p$ and $q$ is defined as $\KL(p\|q)\doteq\E_{X\sim p}\qty[\log\frac{p(X)}{q(X)}]$. Therefore, the KL-divergence between the trajectory distribution of the target transition $p_\omega$ and the simulator's transition $p_{\omega_0}$ is given by 
\begin{align}
\KL(\Pr(H|p_{\omega_0})\|\Pr(H|p_{\omega}))=&\E_{H\sim p_{\omega_0},\pi_\theta}\qty[\log \frac{\Pr(H|p_{\omega_0})}{\Pr(H|p_{\omega})}]\\
=&\E_{H\sim p_{\omega_0},\pi_\theta}\qty[\log \frac{m_{p_{\omega_0}}(H)}{m_{p_{\omega}}(H)}]\explain{By \eqref{eq: PrH decompose}}\\
=&\E_{H\sim p_{\omega_0},\pi_\theta}\qty[\log m_{p_{\omega_0}}(H)-\log m_{p_{\omega}}(H)].
\end{align}
We take the gradient of the KL-divergence with respect to $\omega$:
\begin{align}
\gradomega \KL(\Pr(H|p_{\omega_0})\|\Pr(H|p_{\omega}))=&\gradomega \E_{H\sim p_{\omega_0},\pi_\theta}\qty[\log m_{p_{\omega_0}}(H)-\log m_{p_{\omega}}(H)]\\
=& \E_{H\sim p_{\omega_0},\pi_\theta}\qty[-\gradomega \log m_{p_{\omega}}(H)]\\
=& \E_{H\sim p_{\omega_0},\pi_\theta}\qty[-\sum_{t=0}^{T-1}\gradomega \log p_{\omega}(S_{t+1}|S_t,A_t)]\\
=& \E_{H\sim p_{\omega_0},\pi_\theta}\qty[-\sum_{t=0}^{T-1}\gradomega\log (p_\omega(S_{t+1}|S_t,A_t))].\label{eq: off transition kl gradient}\\
\end{align}
Thus,
\begin{align}
&\gradomega \V_{H\sim p_{\omega_0},\pi_\theta}[\mathrm{OPE}(\pi_e, \pi_\theta,p_\omega,H)]-\eta \KL(\Pr(H|p_{\omega_0})\|\Pr(H|p_{\omega}))\\
=&\textstyle2\E_{\Hsampleoff}\qty[\mathrm{OPE}^2(\pi_e, \pi_\theta,p_\omega,H)\gradomega\ell_{p_\omega} ]\\
&\textstyle-2\E_{\Hsampleoff}\qty[\mathrm{OPE}(\pi_e, \pi_\theta,p_\omega,H)]\textstyle\E_{\Hsampleoff}\qty[\mathrm{OPE}(\pi_e, \pi_\theta,p_\omega,H)\gradomega\ell_{p_\omega} ]\\
&-\textstyle\gradomega\eta \KL(\Pr(H|p_{\omega_0})\|\Pr(H|p_{\omega}))
\\
=&\textstyle2\E_{\Hsampleoff}\qty[\mathrm{OPE}^2(\pi_e, \pi_\theta,p_\omega,H)\gradomega\ell_{p_\omega} ]\\
&\textstyle-2\E_{\Hsampleoff}\qty[\mathrm{OPE}(\pi_e, \pi_\theta,p_\omega,H)]\textstyle\E_{\Hsampleoff}\qty[\mathrm{OPE}(\pi_e, \pi_\theta,p_\omega,H)\gradomega\ell_{p_\omega} ]\\
&-\textstyle\eta \E_{H\sim p_{\omega_0},\pi_\theta}\qty[-\sum_{t=0}^{T-1}\gradomega\log (p_\omega(S_{t+1}|S_t,A_t))].\explain{By \eqref{eq: off transition kl gradient}}
\end{align}

\subsection{Proof of Lemma~\ref{lemma: lipschitz and smoothness of V}}
\label{appendix: lipschitz and smoothness of V}
\begin{proof}
By Lemma~\ref{lemma: variance gradient expression}, 
\begin{align}
\textstyle\frac{\partial}{\partial\theta}\V_{\Hsample}[\IS(\pi_{e},\pi_\theta,H)]=&\E_{\Hsample}\qty[-\IS(\pi_{e},\pi_\theta, H)^2\sum_{t=0}^{T-1}\frac{\partial}{\partial\theta}\log \pi_{\theta}(A_t|S_t)]. \label{eq: policy gradient on variance}
\end{align}
To prove the Lipschitz property, we bound each term in the RHS. First,
we aim to bound
$\norm{\frac{\partial}{\partial\theta}\log \pi_{\theta}(A_t|S_t)}$.
Remember that we define
\begin{align}
\pi_\theta(a|s)\doteq\frac{\exp(\theta^\top_a\phi(s))}{\sum_{a'\in \mathcal{A}}\exp (\theta^\top_{a'}\phi(s))}, 
\end{align}
where we assumed the linear features $\norm{\phi(s)}$ to be bounded by a constant $B$. Here, $\theta=\{\theta_{a}\}_{a\in\mathcal{A}}$ is the whole parameter matrix, and $\theta_a$ is the column for action $a$ specifically.
From \citet{wang2023policy},  we know that
\begin{align}
\norm{\frac{\partial}{\partial\theta}\log \pi_{\theta}(a|s)}^2=\sum_{a'\in \mathcal{A}}\norm{\frac{\partial}{\partial\theta_{a'}}\log \pi_{\theta}(a|s)}^2.
\end{align}
Further decomposing, we get
\begin{align}
\norm{\frac{\partial}{\partial\theta}\log \pi_{\theta}(a|s)}
=&\qty[\norm{\phi(s)}^2_2\qty(1-2\pi_{\theta}(a|s)+\sum_{a'\in \mathcal{A}}\pi_\theta(a'|s)^2)]^{\frac{1}{2}}\\
\leq &\sqrt{2}B.
\end{align}
Thus, we have
\begin{align}
\norm{\sum_{t=0}^{T-1} \frac{\partial}{\partial\theta}\log \pi_{\theta}(A_t|S_t)}\leq \sqrt{2}BT.\label{eq: bound sum log pi}
\end{align}We also make the standard assumption that the quotient $\frac{\pi_e(a|s)}{\pi_\theta(a|s)}$ is bounded above by a constant $C$ for all $s$, $a$, and $\theta$.
\begin{align}
\norm{\IS(\pi_{e},\pi_\theta, H)^2}=\norm{\qty(\frac{\prod_{t=0}^{T-1}\pi_e(A_t|S_t)}{\prod_{t=0}^{T-1}\pi_\theta(A_t|S_t)}g(H) )^2 } \leq C^{2T}T^2, \label{eq bound IS 2}
\end{align}
since we assume the reward is bounded above by $1$.

Then, 
\begin{align}
\norm{\frac{\partial}{\partial\theta}V_{\Hsample}[\IS(\pi_{e},\pi_\theta,H)]}\leq  \sqrt{2}BC^{2T}T^3.
\end{align}
Thus, the objective function $\V_{\Hsample}[\IS(\pi_{e},\pi_\theta,H)]$ is $L_\Theta$-Lipschitz in $\theta$ with $L_\Theta= \sqrt{2}BC^{2T}T^3$.

Next, we aim to show that the objective function $\V_{\Hsample}[\IS(\pi_{e},\pi_\theta,H)]$ is $\ell_\Theta$-smooth in $\theta$. 

Under trajectory $h$,  define $w_{\pi_\theta}(h)\doteq\prod _{t=0}^{T-1} \pi_\theta(A_t,S_t)$ , and $\tilde{p}(h)=\frac{\Pr(H=h|\pi_{\theta})}{w_{\pi_\theta}(h)}$. For a fixed transition $p_\omega$, we have the following decomposition as also shown in \citet{hanna2024data}:
\begin{align}
&\frac{\partial^2}{\partial\theta^2}\V_{\Hsample}[\ISfunction]\\
=&\frac{\partial}{\partial \theta}\E_{\Hsample}\qty[-\IS(\pi_{e},\pi_\theta, H)^2\sum_{t=0}^{T-1}\frac{\partial}{\partial\theta}\log \pi_{\theta}(A_t|S_t)]\explain{By \eqref{eq: policy gradient on variance}}\\
=&\frac{\partial}{\partial \theta}\sum_h \tilde{p}(h)w_{\pi_\theta}(h)\qty(-\IS(\pi_{e},\pi_\theta, H)^2\frac{\partial}{\partial\theta}w_{\pi_\theta}(h)\frac{1}{w_{\pi_\theta}(h)})\\
=&\frac{\partial}{\partial \theta}\sum_h -\tilde{p}(h)\IS(\pi_{e},\pi_\theta, H)^2\frac{\partial}{\partial\theta}w_{\pi_\theta}(h)\\
=&\sum_h -\tilde{p}(h)\qty[\frac{\partial}{\partial \theta} \IS(\pi_{e},\pi_\theta, H)^2\frac{\partial}{\partial\theta}w_{\pi_\theta}(h)+\IS(\pi_{e},\pi_\theta, H)^2\frac{\partial^2}{\partial \theta^2}w_{\pi_\theta}(h)].\label{eq: second derivative V theta}
\end{align}

For the terms here, we have 
\begin{align}
 \frac{\partial}{\partial \theta} \IS(\pi_{e},\pi_\theta, H)^2=\frac{-2g(h)^2w_{\pi_e}(h)^2}{w_{\pi_\theta}(h)^3}   \frac{\partial}{\partial \theta} w_{\pi_\theta}(h),
\end{align}
\begin{align}
\frac{\partial}{\partial \theta} w_{\pi_\theta}(h)=&   \sum_{t=0}^{T-1}\frac{\partial}{\partial \theta} \pi_\theta(A_t|S_t)\prod_{t'=0,t'\neq t}^{T-1}\pi_\theta(A_{t'}|S_{t'}),
\end{align}
and
\begin{align}
\frac{\partial^2}{\partial \theta^2}w_{\pi_\theta}(h)=&\frac{\partial}{\partial \theta} \sum_{t=0}^{T-1}\qty(\frac{\partial}{\partial \theta} \pi_\theta(A_t|S_t)\prod_{t'=0,t'\neq t}^{T-1}\pi_\theta(A_{t'}|S_{t'}))\\
=&\sum_{t=0}^{T-1}\qty(\frac{\partial^2}{\partial \theta^2}\pi_\theta(A_t|S_t)\prod_{t\neq t'}\pi_\theta(A_{t'}|S_{t'})+\frac{\partial}{\partial \theta} \pi_\theta(A_t|S_t)\sum_{t\neq t'}\frac{\partial}{\partial \theta} \pi_\theta(A_{t'}|S_{t'})\prod_{t''\neq t,t'}\pi_\theta(A_{t''}|S_{t''})).\label{eq: second derivative w}
\end{align}

Denote $\theta_\alpha=\theta+\alpha u$, where $\alpha\in \mathbb{R}$, $u\in \mathbb{R}^{d|\mathcal{A}|}$, with $d$ being the linear feature dimension.

By chain rule, we have, for a fixed transition $p_\omega$,
\begin{align}
&\frac{\partial^2}{\partial \alpha^2}    \V_{H\sim \pi_{\theta_\alpha},p_\omega}[\IS(\pi_e, \pi_{\theta_\alpha},H)]\Big|_{\alpha=0}\\
=&u^\top\frac{\partial^2}{\partial \theta^2}    \V_{H\sim \pi_{\theta},p_\omega}[\IS(\pi_e, \pi_{\theta},H)]u\\
=&u^\top \sum_h-\tilde{p}(h)\qty[\frac{\partial}{\partial \theta} \IS(\pi_{e},\pi_\theta, H)^2\frac{\partial}{\partial\theta}w_{\pi_\theta}(h)^\top+\IS(\pi_{e},\pi_\theta, H)^2\frac{\partial^2}{\partial \theta^2}w_{\pi_\theta}(h)] u\explain{By \eqref{eq: second derivative V theta}}\\
=& \sum_h-\tilde{p}(h)\qty[\Big\langle\frac{\partial}{\partial \theta} \IS(\pi_{e},\pi_\theta, H)^2, u\Big\rangle \Big\langle \frac{\partial}{\partial\theta}w_{\pi_\theta}(h), u\Big \rangle+\IS(\pi_{e},\pi_\theta, H)^2u^\top\frac{\partial^2}{\partial \theta^2}w_{\pi_\theta}(h) u].\label{eq: second derivative V theta u}
\end{align}

We analyze the bound term by term. First, for $
\Big\langle\frac{\partial}{\partial \theta} \IS(\pi_{e},\pi_\theta, H)^2, u\Big\rangle$, note that
\begin{align}
&\norm {\frac{\partial}{\partial \theta} \IS(\pi_{e},\pi_\theta, H)^2}\\
=&\norm{\frac{2g(h)^2w_{\pi_e}(h)^2}{w_{\pi_\theta}(h)^3}   \frac{\partial}{\partial \theta} w_{\pi_\theta}(h)}\\
\leq & 2T^2\norm{\frac{w_{\pi_e}(h)^2}{w_{\pi_\theta}(h)^3} w_{\pi_\theta}(h)  \sum_{t=0}^{T-1}\frac{\partial}{\partial \theta} \log(\pi_\theta(A_t|S_t))}\explain{By (21) of \citet{hanna2024data}}\\
\leq& 2T^2C^{2T}\norm{\sum_{t=0}^{T-1}\frac{\partial}{\partial \theta} \log(\pi_\theta(A_t|S_t))}\\
\leq&  2\sqrt{2}BT^3C^{2T}.\explain{By \eqref{eq: bound sum log pi}}
\end{align}
Thus,
\begin{align}
\abs{\Big\langle\frac{\partial}{\partial \theta} \IS(\pi_{e},\pi_\theta, H)^2, u\Big\rangle}\leq   2\sqrt{2}BT^3C^{2T} \norm{u}_2. \label{eq: bound first derivative IS 2 u}
\end{align}
Next, for $\Big\langle \frac{\partial}{\partial\theta}w_{\pi_\theta}(h), u\Big \rangle$, recall that
\begin{align}
&\norm{ \frac{\partial}{\partial\theta}w_{\pi_\theta}(h)}   \\ 
=&\norm{w_{\pi_\theta}(h)  \sum_{t=0}^{T-1}\frac{\partial}{\partial \theta} \log(\pi_\theta(A_t|S_t))}\\
\leq & 1\cdot\sqrt{2}BT.
\end{align}
Thus,
\begin{align}
\abs{\Big\langle \frac{\partial}{\partial\theta}w_{\pi_\theta}(h), u\Big \rangle}   \leq \sqrt{2}BT\norm{u}_2.\label{eq: bound first derivative w}
\end{align}
As for the second term $\IS(\pi_{e},\pi_\theta, H)^2u^\top\frac{\partial^2}{\partial \theta^2}w_{\pi_\theta}(h) u$, remember that by \eqref{eq bound IS 2},
\begin{align}
\norm{\IS(\pi_{e},\pi_\theta, H)^2} \leq C^{2T}T^2.\label{eq: bound IS square}
\end{align}
Besides, by \eqref{eq: second derivative w},
\begin{align}
u^\top\frac{\partial^2}{\partial \theta^2}w_{\pi_\theta}(h)u
=&\underbrace{\sum_{t=0}^{T-1}u^\top\frac{\partial^2}{\partial \theta^2}\pi_\theta(A_t|S_t)u\prod_{t\neq t'}\pi_\theta(A_{t'}|S_{t'})}_{(a)}\\
&+\underbrace{\sum_{t=0}^{T-1}\Big\langle\frac{\partial}{\partial \theta} \pi_\theta(A_t|S_t),u\Big\rangle \cdot\sum_{t\neq t'}\Big\langle\frac{\partial}{\partial \theta} \pi_\theta(A_{t'}|S_{t'}),u\Big\rangle \prod_{t''\neq t,t'}\pi_\theta(A_{t''}|S_{t''})}_{(b)}.\label{eq: second derivative theta u}
\end{align}
 

To bound term (a) and (b), notice that
\begin{align}
\frac{\partial\pi_{\theta}(a|s)}{\partial \theta_a'}=\pi_{\theta}(a|s)(\mathbf{1}\{a'=a\}-\pi_\theta( a'|s))\phi(s)  ,
\end{align}
where $\mathbf{1}$ is the indicator function. Now we also define a state-wise logit direction $v_s\in \mathbb{R}^{|\mathcal{A}|}$ with each component $v_s(a')\doteq \langle u_{a'},\phi(s)\rangle$.\\
Then, for the first derivative,
\begin{align}
\abs{\frac{\partial \pi_{\theta_\alpha}(a|s)}{\partial \alpha}\Big |_{\alpha=0} }=&\abs{\Big\langle\frac{\partial\pi_{\theta}(a|s)}{\partial \theta} ,u\Big\rangle   }\\
=&\abs{\pi_{\theta}(a|s)\cdot\qty(v_s(a)-\sum_{a'}\pi_\theta(a'|s)v_s(a'))}\\
\leq &\pi_{\theta}(a|s)\qty(\abs{v_s(a)}+\abs{\sum_{a'}\pi_\theta(a'|s)v_s(a')})\explain{Triangular Inequality}\\
\leq &2\pi_{\theta}(a|s)\norm{v_s}_2\\
\leq &2\norm{v_s}_2\\
\leq  &2 \norm{\phi(s)}_2\norm{u}_2\\
\leq & 2B\norm{u}_2, \label{eq: first derivative bound 2}
\end{align}
where we identify $u$ with its vectorization, so $\norm{u}^2_2=\sum_{a\in \mathcal{A}}\norm{u_a}^2_2$.
Similarly, for the second derivative,
\begin{align}
\abs{\frac{\partial^2 \pi_{\theta_\alpha}(a|s)}{\partial \alpha^2}\Big |_{\alpha=0} }=&\abs{\Big\langle\frac{\partial^2 \pi_{\theta}(a|s)}{\partial \theta^2}v_s ,v_s\Big\rangle   }\\
=&\abs{\pi_\theta(a\mid s)\qty[
(1-\pi_\theta(a\mid s))v_s(a)^2
-\sum_{a'\neq a}\pi_\theta(a'\mid s)\bigl(v_s(a)-v_s(a')\bigr)^2
+\sum_{a'}\pi_\theta(a'\mid s)^{2}v_s(a')^2
]}\\
\leq &5\norm{v_s}^2_2\\
\leq & 5 B^2\norm{u}^2_2\label{eq: second derivative bound 5}
\end{align}

Now, getting back to (a) in \eqref{eq: second derivative theta u}, we have
\begin{align}
\abs{(a)}=&\abs{\sum_{t=0}^{T-1}u^\top\frac{\partial^2}{\partial \theta^2}\pi_\theta(A_t|S_t)u\prod_{t\neq t'}\pi_\theta(A_{t'}|S_{t'})}\\
\leq &\abs{\sum_{t=0}^{T-1}u^\top\frac{\partial^2}{\partial \theta^2}\pi_\theta(A_t|S_t)u}\\
\leq &5TB^2\norm{u}^2_2\explain{By \eqref{eq: second derivative bound 5}}.\\
\end{align}
As for (b) in \eqref{eq: second derivative theta u},
\begin{align}
\abs{(b)}=&\abs{\sum_{t=0}^{T-1}\Big\langle\frac{\partial}{\partial \theta} \pi_\theta(A_t|S_t),u\Big\rangle \cdot\sum_{t\neq t'}\Big\langle\frac{\partial}{\partial \theta} \pi_\theta(A_{t'}|S_{t'}),u\Big\rangle \prod_{t''\neq t,t'}\pi_\theta(A_{t''}|S_{t''})}\\
\leq &\sum_{t=0}^{T-1}\abs{\Big\langle\frac{\partial}{\partial \theta} \pi_\theta(A_t|S_t),u\Big\rangle }\cdot\sum_{t\neq t'}\abs{\Big\langle\frac{\partial}{\partial \theta} \pi_\theta(A_{t'}|S_{t'}),u\Big\rangle }\cdot1\\
\leq &2TB\norm{u}_2\cdot 2TB\norm{u}_2\explain{By \eqref{eq: first derivative bound 2}}\\
=&4T^2B^2\norm{u}^2_2.
\end{align}
Thus, looking back at the \eqref{eq: second derivative theta u}, we have
\begin{align}
&\abs{u^\top\frac{\partial^2}{\partial \theta^2}w_{\pi_\theta}(h)u}\\
\leq &\abs{(a)}+\abs{(b)}\\
\leq & 5TB^2\norm{u}^2_2+4T^2B^2\norm{u}^2_2.\\
\end{align}
Therefore, 
\begin{align}
&\abs{\IS(\pi_{e},\pi_\theta, H)^2u^\top\frac{\partial^2}{\partial \theta^2}w_{\pi_\theta}(h) u}\\
\leq &\abs{\IS(\pi_{e},\pi_\theta, H)^2}\cdot\abs{u^\top\frac{\partial^2}{\partial \theta^2}w_{\pi_\theta}(h) u}\\
\leq &C^{2T}T^3B^2\norm{u}^2_2 ( 5+4T).\label{eq: bound long second term}\\
\end{align}
Putting these all together,
\begin{align}
&\abs{\frac{\partial^2}{\partial \alpha^2}    \V_{H\sim p_\omega, \pi_{\theta_\alpha}}[\IS(\pi_e, \pi_{\theta_\alpha},H)]\Big|_{\alpha=0}}\\
=&\abs{  \sum_h-\tilde{p}(h)\qty[\Big\langle\frac{\partial}{\partial \theta} \IS(\pi_{e},\pi_\theta, H)^2, u\Big\rangle \Big\langle \frac{\partial}{\partial\theta}w_{\pi_\theta}(h), u\Big \rangle+\IS(\pi_{e},\pi_\theta, H)^2u^\top\frac{\partial^2}{\partial \theta^2}w_{\pi_\theta}(h) u]}\explain{By \eqref{eq: second derivative V theta u}}\\
\leq &\abs{\Big\langle\frac{\partial}{\partial \theta} \IS(\pi_{e},\pi_\theta, H)^2, u\Big\rangle \Big\langle \frac{\partial}{\partial\theta}w_{\pi_\theta}(h), u\Big \rangle}+\abs{\IS(\pi_{e},\pi_\theta, H)^2u^\top\frac{\partial^2}{\partial \theta^2}w_{\pi_\theta}(h) u}\\
\leq & \qty(2\sqrt{2}BT^3C^{2T} \norm{u}_2)\cdot\qty(
 \sqrt{2}BT\norm{u}_2
) + C^{2T}T^3B^2\norm{u}^2_2 ( 5+4T)\explain{By \eqref{eq: bound first derivative IS 2 u}\eqref{eq: bound first derivative w}\eqref{eq: bound long second term}}\\
=&4B^2C^{2T}T^4\norm{u}^2_2+ C^{2T}T^3B^2\norm{u}^2_2\cdot ( 5+4T)\\
=&B^2C^{2T}T^3 \norm{u}^2_2 \qty( 5+8T).
\end{align}
Thus, 
\begin{align}
\norm{\frac{\partial^2}{\partial \theta^2}    \V_{H\sim p_\omega, \pi_{\theta}}[\IS(\pi_e, \pi_{\theta},H)]}_{\mathrm{op}}=&\sup_{\norm{u}_2=1}B^2C^{2T}T^3 \norm{u}^2_2 \qty( 5+8T)\\
=&B^2C^{2T}T^3 \qty(5+8T),
\end{align}
where $\norm{\cdot}_{\mathrm{op}}$ denotes the operator norm. Therefore, we conclude that the objective function $\V_{\Hsample}[\IS(\pi_{e},\pi_\theta,H)]$ is $\ell_\Theta$-smooth in $\theta$ with $\ell_\Theta= B^2C^{2T}T^3 (5+8T)$. \\
Lastly, the convexity of the objective function follows directly from Lemma 2 of \citet{hanna2024data}.
\end{proof}
\subsection{Proof of Lemma~\ref{lemma: properties Phi}}
\label{appendix: properties Phi}
\begin{proof}
We first show that  $\Phi(\theta)$ is $L_\Theta$-Lipschitz in $\theta$. By Lemma~\ref{lemma: lipschitz and smoothness of V}, we know that $\V_{\Hsample}[\IS(\pi_{e},\pi_\theta,H)]$ is $L_\Theta$-Lipschitz. \\
$\forall \theta_1$, $\theta_2\in \Theta$, define $p_{\omega_1}\doteq\argmax_{p_\omega} \V_{H\sim p_\omega,\pi_{\theta_1}}[\IS(\pi_{e},\pi_{\theta_1},H)]$ , $p_{\omega_2}\doteq\argmax_{p_\omega} \V_{H\sim p_\omega,\pi_{\theta_2}}[\IS(\pi_{e},\pi_{\theta_2},H)]$. Then,
\begin{align}
\Phi(\theta_1)-\Phi(\theta_2)=&\max_{p_\omega}\V_{H\sim p_\omega,\pi_{\theta_1}}[\IS(\pi_{e},\pi_{\theta_1},H)]-\max_{p_{\omega}} \V_{H\sim p_\omega,\pi_{\theta_2}}[\IS(\pi_{e},\pi_{\theta_2},H)]\\
=&\V_{H\sim p_{\omega_1},\pi_{\theta_1}}[\IS(\pi_{e},\pi_{\theta_1},H)]-\V_{H\sim p_{\omega_2},\pi_{\theta_2}}[\IS(\pi_{e},\pi_{\theta_2},H)]\\
\leq&\V_{H\sim p_{\omega_1},\pi_{\theta_1}}[\IS(\pi_{e},\pi_{\theta_1},H)]-\V_{H\sim p_{\omega_1},\pi_{\theta_2}}[\IS(\pi_{e},\pi_{\theta_2},H)]\\
\leq & L_\Theta\norm{\theta_1-\theta_2}.\explain{By Lemma~\ref{lemma: lipschitz and smoothness of V}} 
\end{align}
By symmetry, with also have 
\begin{align}
\Phi(\theta_2)-\Phi(\theta_1)\leq  L_\Theta\norm{\theta_1-\theta_2}.
\end{align}
Thus, 
\begin{align}
\abs{\Phi(\theta_1)-\Phi(\theta_2)}\leq  L_\Theta\norm{\theta_1-\theta_2},
\end{align}
which shows the Lipschitz property. \\
Next, from Lemma~\ref{lemma: lipschitz and smoothness of V}, we also know that $\V_{\Hsample}[\IS(\pi_{e},\pi_\theta,H)]$ is convex in $\theta$ under the linear softmax parameterization of the behavior policy $\pi_\theta$. Thus, $\forall \theta_1$, $\theta_2\in \Theta$ and $t\in [0,1]$, 
\begin{align}
\Phi(t\theta_1+(1-t)\theta_2)=&\max_{p_\omega}\V_{H\sim p_\omega,\pi_{(t\theta_1+(1-t)\theta_2)}}[\IS(\pi_{e},\pi_{(t\theta_1+(1-t)\theta_2)},H)]\\
\leq &\max_{p_\omega}[t\V_{H\sim p_\omega,\pi_{\theta_1}}[\IS(\pi_{e},\pi_{\theta_1},H)]+(1-t)\V_{H\sim p_\omega,\pi_{\theta_2}}[\IS(\pi_{e},\pi_{\theta_2},H)]] \explain{By Lemma~\ref{lemma: lipschitz and smoothness of V}}\\
\leq& t\max_{p_\omega}[\V_{H\sim p_\omega,\pi_{\theta_1}}[\IS(\pi_{e},\pi_{\theta_1},H)]+(1-t) \max_{p_{\omega'}}\V_{H\sim p_{\omega'},\pi_{\theta_2}}[\IS(\pi_{e},\pi_{\theta_2},H)]\\
=& t\Phi(\theta_1)+(1-t)\Phi(\theta_2).
\end{align}
Therefore, we show that $\Phi(\theta)$ is convex in $\theta$.
\end{proof}

\subsection{Proof of Theorem~\ref{theorem: double loop convergence}}
\label{appendix: proof of double loop convergence}
\begin{proof}
To begin with, we define $\theta^*\doteq\argmin_{\theta\in \Theta}\Phi(\theta)$. 
Since the set $\Theta$ is closed and convex, the Euclidean projection is nonexpensive. That is, $\forall u\in \mathbb{R}^d, z\in \Theta$, 
\begin{align}
\norm{\mathrm{Proj}_\Theta(u)-z}^2\leq \norm{u-z}^2.
\end{align}
With $u\doteq \theta_i-\alpha \mathcal{G}_i$, $z\doteq \theta^*$, we have 
\begin{align}
 \mathrm{Proj}_\Theta(u)=\theta_{i+1}   .
\end{align}
Thus,
\begin{align}
\norm{\theta_{i+1}-\theta^*}^2&\leq \norm{\theta_i-\alpha \mathcal{G}_i-\theta^*}^2\\
&=\norm{\theta_i-\theta^*}^2-2\alpha\langle \mathcal{G}_i,\theta_i-\theta^*\rangle+\alpha^2\norm{\mathcal{G}_i}^2.\label{eq: long expectation breakdown}
\end{align}
From here, we first bound the last term, $\norm{\mathcal{G}_i}^2$. By Lemma~\ref{lemma: lipschitz and smoothness of V} we know that the objective function $\V_{\Hsample}[\IS(\pi_{e},\pi_\theta,H)]$ is $L_\Theta$-Lipschitz and convex in $\theta$. Thus we have 
\begin{align}
\norm{\mathcal{G}_i}\leq L_\Theta \implies \norm{\mathcal{G}_i}^2\leq L^2_\Theta. \label{eq: gradient square bound}
\end{align}
Next, since the gradient objective $\V_{\Hsample}[\IS(\pi_{e},\pi_\theta,H)]$ is differentiable and convex in $\theta$, we have the subgradient inequality that
\begin{align}
\Big\langle\mathcal{G}_i, \theta_i-\theta^*\Big\rangle \geq \V_{H\sim p_{\omega_i}, \pi_{\theta_i}}[\IS(\pi_{e},\pi_{\theta_i},H)]-\V_{H\sim p_{\omega_i}, \pi_{\theta^*}}[\IS(\pi_{e},\pi_{\theta^*},H)].
\end{align}
Remember that we defined
\begin{align}
\Phi(\theta)\doteq \max_{p_\omega}\V_{\Hsample}[\IS(\pi_{e},\pi_\theta,H)] .
\end{align}
Thus, 
\begin{align}
\V_{H\sim p_{\omega_i}, \pi_{\theta^*}}[\IS(\pi_{e},\pi_{\theta^*},H)]\leq \Phi(\theta^*),
\end{align}and by Algorithm~\ref{alg:  double with inner oracle},
\begin{align}
\max_p\V_{H\sim p, \pi_{\theta_i}}[\IS(\pi_{e},\pi_{\theta_i},p,H)]=\Phi(\theta_i)\leq \V_{H\sim p_{\omega_i}, \pi_{\theta_i}}[\IS(\pi_{e},\pi_{\theta_i},H)]+\epsilon_i.
\end{align}
Therefore,
\begin{align}
\Big\langle\mathcal{G}_i, \theta_i-\theta^*\Big\rangle \geq \Phi(\theta_i)-\epsilon_i-\Phi(\theta^*).\label{eq: bound phi epsilon}
\end{align}
Putting it all together, by \eqref{eq: long expectation breakdown}, \eqref{eq: gradient square bound}, and \eqref{eq: bound phi epsilon}, we have
\begin{align}
\norm{\theta_{i+1}-\theta^*}^2\leq & \norm{\theta_i-\theta^*}^2-2\alpha   \qty(\Phi(\theta_i)-\epsilon_i-\Phi(\theta^*))+\alpha^2L^2_\Theta.
\end{align}
Rearranging the terms, we get
\begin{align}
2\alpha\qty[\Phi(\theta_i)-\Phi(\theta^*)]\leq \norm{\theta_i-\theta^*}^2-\norm{\theta_{i+1}-\theta^*}^2+2\alpha\epsilon_i+\alpha^2L^2_\Theta.
\end{align}
Taking the summation over $i$,
\begin{align}
2\alpha\sum_{i=0}^{n-1}\Phi(\theta_i)-\Phi(\theta^*)\leq \norm{\theta_0-\theta^*}^2+2\alpha\sum_{i=0}^{n-1}\epsilon_i+n\alpha^2L^2_\Theta.
\end{align}
Since $\theta_0,\theta^*\in\Theta$, and we defined $\mathrm{diam}(\Theta)\leq D$ where 
\begin{align}
 \mathrm{diam}(\Theta)\doteq\max_{\theta,\theta'\in \Theta}\norm{\theta-\theta'},   
\end{align}we have 
\begin{align}
 \norm{\theta_0-\theta^*}^2\leq D^2.   
\end{align}
Thus, 
\begin{align}
\frac{1}{n}\sum_{i=0}^{n-1}\Phi(\theta_i)-\Phi(\theta^*)\leq\frac{D^2}{2\alpha n}+\frac{\alpha L^2_\Theta}{2}+\frac{1}{n}\sum_{i=0}^{n-1}\epsilon_i.\label{eq: bound 1/n Phi}
\end{align}

According to Algorithm~\ref{alg:  double with inner oracle}, 
\begin{align}
\bar \theta\doteq \frac{1}{n}\sum_{i=0}^{n-1}\theta_i.
\end{align}
By Lemma~\ref{lemma: properties Phi}, $\Phi(\pi_\theta)$ is convex in $\theta$. Thus, by induction with the basic convex property, with the nonnegative weight $\frac{1}{n}$ and the fact that $\sum_{i=0}^{n-1}=1$, we obtain
\begin{align}
\Phi(\bar\theta)=\Phi\qty(\frac{1}{n}\sum_{i=0}^{n-1}\theta_i)\leq \frac{1}{n}\sum_{i=0}^{n-1}\Phi(\theta_i).
\end{align}
Subtracting $\Phi(\theta^*)$ form both sides, we get
\begin{align}
\Phi(\bar \theta)-\Phi(\theta^*)\leq  \frac{1}{n}\sum_{i=0}^{n-1}\Phi(\theta_i)-\Phi(\theta^*).
\end{align}
Plugging it into \eqref{eq: bound 1/n Phi},
\begin{align}
\Phi(\bar \theta)-\Phi(\theta^*)\leq \frac{D^2}{2\alpha n}+\frac{\alpha L^2_\Theta}{2}+\frac{1}{n}\sum_{i=0}^{n-1}\epsilon_i.
\end{align}
With the definition $\theta^*\doteq\argmin_{\theta\in \Theta}\Phi(\theta)$ and $\alpha\doteq\frac{D}{L_\Theta \sqrt{n}}$, we then have
\begin{align}
\Phi(\bar \theta)-\min_{\theta\in \Theta}\Phi(\theta)\leq \frac{DL_\Theta}{\sqrt{n}}+\frac{1}{n}\sum_{i=0}^{n-1}\epsilon_i.
\end{align}
\end{proof}
\clearpage

\section{Numerical Studies}
\label{appendix: experiments}
\begin{figure}[h]
\includegraphics[width=1\textwidth]{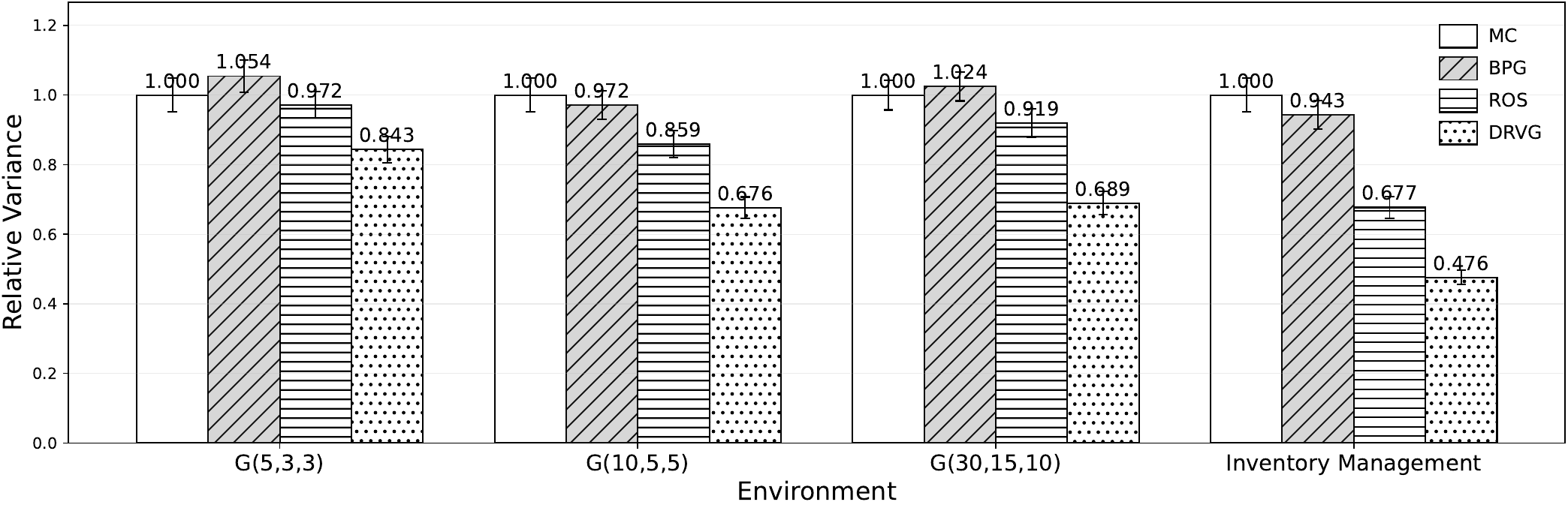}
\centering
\caption{ Supplementary figure for Section~\ref{sec: tailored adversary}.
Relative variance of each method under its tailored adversarial transition. All values are normalized by the variance of the on-policy Monte Carlo (MC) method (under its tailored adversarial transition) in the same environment. Error bars denote the standard error over 900 total runs per environment 
(30 target policies with 30 independent runs each).
}
\label{fig:variance under its adv compare}
\end{figure} 

In our numerical studies, we leverage a wide range of target policies ranging from completely random to highly deterministic. Specifically, a policy $\pi_{\text{train}}$ is 
computed as the optimal policy of the MDP model, and $\pi_{\text{random}}$ is randomly generated. Then, the target policies $\pi_e$ are set to be $(1-\beta)\pi_{\text{train}}+\beta \pi_{\text{random}}$ with $\beta \in \{\frac{1}{30}, \frac{2}{30}, ... ,1 \}$. For each of the $30$ target policies, we have $30$ independent runs, resulting in a total of $900$ runs for each value.

In Section~\ref{sec: tailored adversary}, we generate method-specific adversarial transitions by running Algorithm~\ref{alg: inner variance algorithm} separately for each method, yielding one adversarial transition per behavior policy. For ROS \citep{zhong2022robust}, which adapts its behavior policy, we generate the adversarial transition by treating the target policy as the nominal behavior policy in Algorithm~\ref{alg: inner variance algorithm}. 
We first measure each method's variance under the original simulator transition $p_0$, and then under its method-specific adversarial transition $p_{\text{adv}}$. Note that $p_{\text{adv}}$ differs across methods. We report the relative variance of each method under its own $p_{\text{adv}}$ in Figure~\ref{fig:variance under its adv compare}, where each value is normalized by the variance of MC (also under its $p_{\text{adv}}$). Finally, we report the variance increase for each method in Figure~\ref{fig:variance increase}, defined as the difference between its variance under $p_{\text{adv}}$ and under $p_0$. In Section~\ref{sec: shared perturbed transition}, we evaluate all methods under a shared adversarial target transition. This transition is constructed by applying Algorithm~\ref{alg: inner variance algorithm} to the on-policy Monte Carlo baseline. We then report each method’s variance under this transition in Figure~\ref{fig:variance compare}.

\subsection{Experimental Setup}

To ensure reproducibility and to isolate the source of variance reduction across methods, we adopt a uniform experimental protocol across all four environments. Our setup follows the standard conventions established in the behavior policy search literature \citep{hanna2017data, hanna2024data, zhong2022robust}.
We parameterize the behavior policy $\pi_\theta$ as a two-layer multilayer perceptron (MLP) with 64 hidden units per layer and tanh activations, followed by a softmax output head over the discrete action space. We adopt the same architecture for the adversarial transition model $p_\omega$, with its softmax head defined over the discrete next-state space. This shared architecture, deployed identically across environments and across all four methods, removes parameterization choices as a source of confounding variance in our comparisons.

We optimize all network parameters with the Adam optimizer \citep{kingma2014adam}, using a learning rate of $10^{-3}$ and the default momentum coefficients $\beta_1 = 0.9$ and $\beta_2 = 0.999$, with related work studying deep-network trainability under continued learning \citep{wang2026predicting}.
For BPG \citep{hanna2017data} and ROS \citep{zhong2022robust}, we adopt the hyperparameters reported in their original publications, ensuring that each baseline is evaluated under conditions favorable to its own design. All four methods (MC, BPG, ROS, and DRVG) are trained on the same initial simulator transition and evaluated against the same set of target policies, so any performance differences are attributable to the variance-reduction strategy itself rather than to environment configuration.

To establish the statistical reliability of our reported numbers, we evaluate each method across 30 target policies generated by the mixing scheme $\pi_e = (1-\beta)\pi_{\text{train}} + \beta \pi_{\text{random}}$ described in Appendix~\ref{appendix: experiments}, with 30 independent runs per target policy. This yields 900 total runs per reported value. Error bars in Figures~\ref{fig:variance increase}, \ref{fig:variance compare}, and \ref{fig:variance under its adv compare} denote the standard error of the mean computed across these 900 runs. 
This protocol provides sufficient resolution to detect consistent variance differences between methods, which is essential given the high baseline variance of policy evaluation under transition perturbations.
Recent work also studies evaluation in scientific and formal reasoning settings \citep{liu2026mathliblemma, chen2026astroalertbench}.

All experiments were conducted on a single shared compute node equipped with two AMD EPYC 7663 processors (56 cores per socket, two sockets, two threads per core, yielding 224 logical CPUs of which 222 are allocatable to jobs) and 1024 GB of system memory (1000 GB allocatable).

\subsection{Garnet Examples}
A Garnet environment \citep{archibald1995generations} is represented by three integers $(|S|, |A|, b)$, denoting the number of states, actions, and the branching factor, respectively. By varying $b$, one controls the degree of stochasticity: small $b$ yields sparse transitions, while large $b$ approaches fully connected transitions. This flexibility makes Garnets particularly suitable for stress-testing reinforcement learning algorithms across a wide spectrum of transition structures \citep{tarbouriech2019active, wang2023policy, wang2023robust}. We evaluate the four methods on three Garnet instances---$G(5,3,3)$, $G(10,5,5)$, and $G(30,15,10)$---which span increasing environment sizes and connectivity levels.

\subsection{Inventory Management}
Inventory management
\citep{porteus2002foundations,ho2018fast, liu2026ortransformer} is a classical stochastic control problem under transition uncertainty. The state corresponds to inventory levels, actions represent order quantities, and stochastic demand drives the state transitions. In our inventory management example, we adopt radial-type basis functions as introduced in \citet{sutton2018reinforcement}, defined for state $s$ and feature index $i$ as $
\phi_i(s) = \exp\!\left(- \frac{\| s - c_i \|^2}{2\sigma_i^2}\right),
$
where $c_i$ and $\sigma_i$ denote the deterministic center and scaling parameter of the $i$-th feature, respectively. 
This nonlinear parameterization captures variations in state representation while controlling the expressive capacity of the model under uncertainty.
\end{document}